\documentclass[journal,twoside,web]{ieeecolor}
\pdfoutput=1
\usepackage{tmi}
\usepackage{cite}
\usepackage{colortbl}
\usepackage{amsmath,amssymb,amsfonts}
\usepackage{graphicx}
\usepackage{textcomp}
\usepackage{wrapfig}
\usepackage{subcaption}
\usepackage{caption}
\usepackage{multirow}
\usepackage{algorithm}
\newtheorem{theorem}{Theorem}
\usepackage{algpseudocode}
\usepackage{color}
\usepackage{tabularray}
\usepackage{comment}
\usepackage{setspace}
\usepackage{amssymb}
\usepackage{comment}
\usepackage{booktabs}
\usepackage{url}

\newcommand{\ie}{\textit{i.e.}}
\newcommand{\eg}{\textit{e.g.}}
\newcommand{\etal}{\textit{et al.}}
\newcommand{\best}[1]{{\textcolor{red}{\textbf{#1}}}}
\newcommand{\secondbest}[1]{{\color{blue}{\textbf{#1}}}}

\def\BibTeX{{\rm B\kern-.05em{\sc i\kern-.025em b}\kern-.08em
    T\kern-.1667em\lower.7ex\hbox{E}\kern-.125emX}}
\begin{document}
\title{Medical Image Debiasing by Learning Adaptive Agreement from a Biased Council}
\author{Luyang~Luo, \IEEEmembership{Member, IEEE}, Xin Huang, Minghao Wang, Zhuoyue Wan, and Hao Chen, \IEEEmembership{Senior Member, IEEE}
\thanks{This work was supported by the Pneumoconiosis Compensation Fund Board, HKSARS (Project No. PCFB22EG01), Hong Kong Innovation and Technology Fund (Project No. ITS/028/21FP), Shenzhen Science and Technology Innovation Committee Fund (Project No. SGDX20210823103201011), and the Project of Hetao Shenzhen-Hong Kong Science and Technology Innovation Cooperation Zone (HZQB-KCZYB-2020083). (Corresponding author: Hao Chen.)}
\thanks{Luyang Luo and Xin Huang are with the Department of Computer Science and Engineering, The Hong Kong University of Science and Technology, Hong Kong, China (e-mail: cseluyang@ust.hk).}
\thanks{Minghao Wang is with the Department of Chemical and Biological Engineering, The Hong Kong University of Science and Technology, Hong Kong, China.}
\thanks{Zhuoyue Wan is with Department of Computing, The Hong Kong Polytechnic University, Hong Kong, China.}
\thanks{Hao Chen is with the Department of Computer Science and Engineering and Department of Chemical and Biological Engineering, The Hong Kong University of Science and Technology, Hong Kong, China. 
Hao Chen is also affiliated with HKUST Shenzhen-Hong Kong Collaborative Innovation Research Institute, Futian, Shenzhen, China. (e-mail: jhc@cse.ust.hk).}
\thanks{}}

\maketitle

\begin{abstract}
Deep learning could be prone to learning shortcuts raised by dataset bias and result in inaccurate, unreliable, and unfair models, which impedes its adoption in real-world clinical applications.  Despite its significance, there is a dearth of research in the medical image classification domain to address dataset bias. Furthermore, the bias labels are often agnostic, as identifying biases can be laborious and depend on post-hoc interpretation. This paper proposes learning Adaptive Agreement from a Biased Council (Ada-ABC), a debiasing framework that does not rely on explicit bias labels to tackle dataset bias in medical images. Ada-ABC develops a biased council consisting of multiple classifiers optimized with generalized cross entropy loss to learn the dataset bias. A debiasing model is then simultaneously trained under the guidance of the biased council. Specifically, the debiasing model is required to learn adaptive agreement with the biased council by agreeing on the correctly predicted samples and disagreeing on the wrongly predicted samples by the biased council. In this way, the debiasing model could learn the target attribute on the samples without spurious correlations while also avoiding ignoring the rich information in samples with spurious correlations. We theoretically demonstrated that the debiasing model could learn the target features when the biased model successfully captures dataset bias. Moreover, to our best knowledge, we constructed the first medical debiasing benchmark from four datasets containing seven different bias scenarios. Our extensive experiments practically showed that our proposed Ada-ABC outperformed competitive approaches, verifying its effectiveness in mitigating dataset bias for medical image classification. The codes and organized benchmark datasets will be released via \url{https://github.com/LLYXC/PBBL}. 
\end{abstract}


\section{Introduction}

Artificial intelligence (AI), typically represented by deep learning, has achieved expert-level performance in many domains of medical image analysis \cite{topol2019high}.
However, the trustworthiness of deep learning models is challenged by their preference of learning from spurious correlations caused by shortcuts, or dataset biases \cite{geirhos2020shortcut,oakden2020hidden}.
Concerns have also been raised in the medical image classification domain that deep models could learn biases other than the targeted features \cite{oakden2020hidden,luo2022pseudo,degrave2021ai,larrazabal2020gender,gichoya2022ai}, leading to misdiagnosis and unfairness for the less-represented groups.
Consequently, there are rising calls to include more evaluation procedures to ensure that a deep learning model is unbiased before deployment as a medical product \cite{bluemke2020assessing,taylor2022uk}.
Hence, mitigating dataset bias to develop trustworthy medical models plays a significant role in facilitating the integration of deep learning into real-world clinical applications.

\begin{figure}[t]
\begin{center}
   \includegraphics[width=\linewidth]{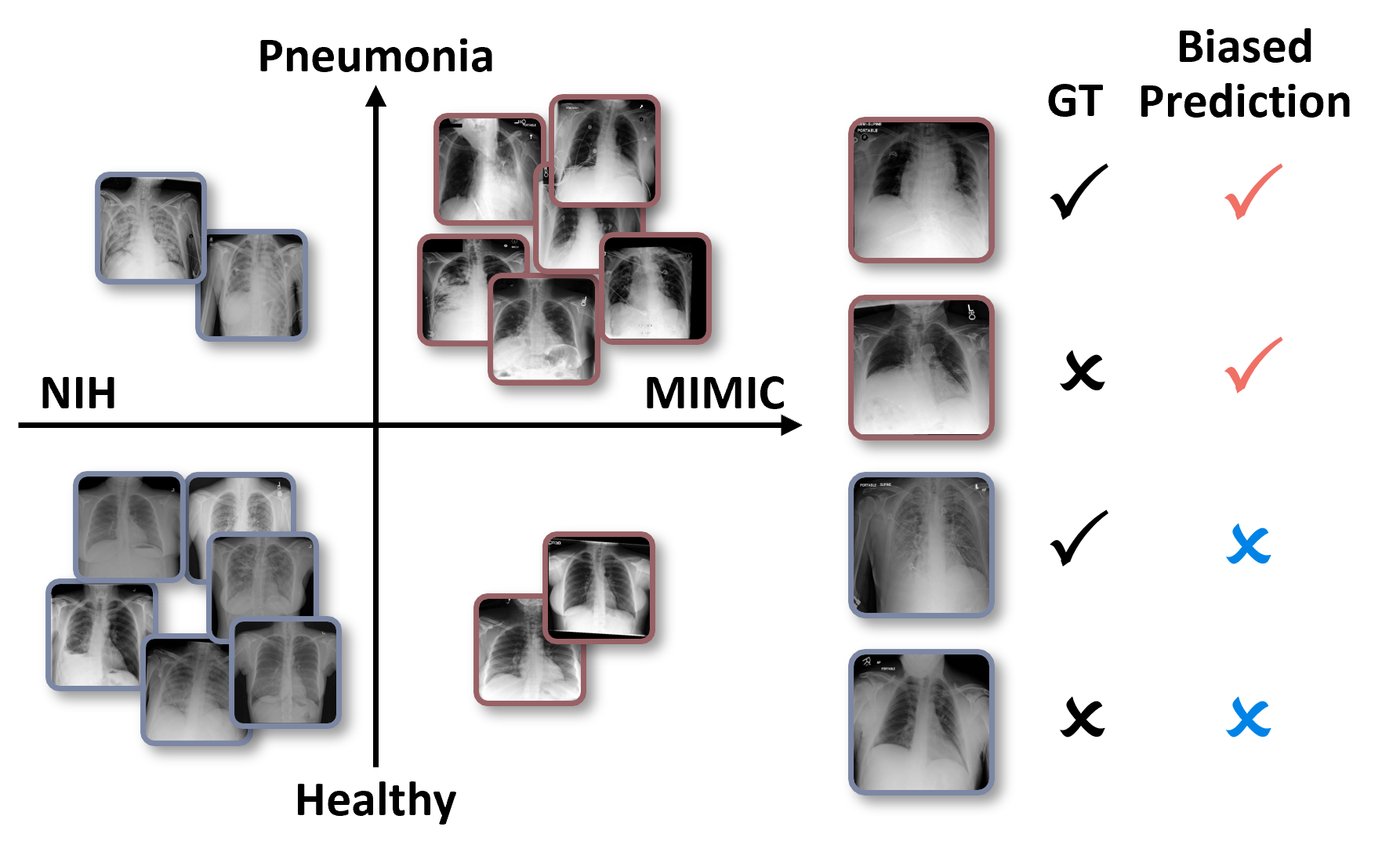}
\end{center}
   \caption{Dataset bias in medical image classification could lead to inaccurate and untrustworthy results. Here, the source of data and whether the patient contains pneumonia are spuriously correlated. A biased model would make decisions based on the data source while ignoring the patterns of the lesions. Our goal is to learn a robust model that can make bias-invariant decisions from the biased training set.}
\label{fig:teaser}
\end{figure}

Specifically, dataset biases, or shortcuts, are often referred to the features that spuriously correlate to the target patterns.
Previous works argue that such features would be preferred by the deep learning models as they are much easier to learn \cite{luo2022pseudo,nam2020learning,shah2020pitfalls}.
In particular, the training of stochastic gradient descent (SGD) tend to find simple solutions, which is also called simplicity bias by some and deemed one of the reasons why deep neural networks generalize well \cite{kalimeris2019sgd,hermann2020shapes,teney2022evading}. 
However, such a preference can also prevent models from learning more complex patterns, providing a precondition for models to use shortcuts to quickly fit the entire training data \cite{shah2020pitfalls,geirhos2020shortcut,nam2020learning}.
Simplicity bias is especially harmful when non-causal factors have spurious correlations with the target patterns.
For example, a convolutional neural network may identify pneumothorax patients based on the patterns of the chest drain \cite{oakden2020hidden,rueckel2020impact}, a common treatment to remove air or fluid from the pleural space.
Compared with the complex and ambiguous signs of pneumothorax, chest drains show clearer patterns on the radiograph and are more easily recognized.
Similar findings have been reported on the preference of learning data source over thoracic diseases (as shown in Fig. \ref{fig:teaser}) \cite{luo2022pseudo}, gender signs over pneumothorax \cite{larrazabal2020gender}, or even laterality markers \footnote{The laterality marker is a sign of "L" or "R" put on a chest radiograph to indicate the side of a patient.} over COVID-19 lesions \cite{degrave2021ai}.
Despite being reported frequently, there are few benchmarks and solutions on mitigating dataset biases in the medical image domain.

Re-collecting data could remove dataset biases, which is conceptually simple but practically infeasible \cite{rouzrokh2022mitigating}.
Therefore, many studies attempted to learn a robust model by up-weighting the minority group of samples (\eg, pneumothorax cases without chest drains) \cite{li2019repair,sagawa2020distributionally}.
Another broad direction proposes learning invariant representations across different environments \cite{arjovsky2019invariant,zhang2021quantifying,zhou2022sparse,tartaglione2021end,zhu2021learning}.
These methods replace the need for data collection with algorithmic solutions using explicit labels of the dataset bias, which are still less practical as the dataset bias is often unknown until careful evaluation and interpretation of the trained model \cite{bluemke2020assessing,degrave2021ai}, and explicitly labeling the dataset bias is tedious and expertise-dependent.

More recent studies explored alleviating shortcut learning without explicit dataset bias labels \cite{sohoni2020no,nam2020learning,liu2021just,lee2021learning,kim2021biaswap,luo2022pseudo}, which can be roughly categorized into two-stage methods and one-stage methods.
The two-stage methods first capture and predict the bias information and then develop group-robust learning models based on the predicted bias information \cite{sohoni2020no,liu2021just,kim2021biaswap,luo2022pseudo}.
However, these approaches are sensitive to the convergence of biased models and could bring in noise during the group-robust learning process.
In contrast, the one-stage methods typically develop debiasing models by comparing its loss with that of a simultaneously trained biased model \cite{nam2020learning,lee2021learning}.
Nevertheless, these approaches are mainly based on heuristic loss weighting functions, which could lead to over-weighing of the samples without spurious correlations and prevent the model from learning the target patterns.
Taking binary classification as an example, let two binary variables $t$ and $b$ represent the target feature and bias feature, respectively.
A biased model could majorly learn from the samples with spurious correlation, \eg, $t=b$, and make decisions according to $b$.
When restricted to only learning from the data without spurious correlation, the debiasing model could probably learn another biased decision, \eg, $t\neq b$, and still make decisions according to $b$.
In this sense, heuristically up-weighing the samples without spurious correlation could be harmful, and a good debias model should learn from both types of samples.

To this end, this paper proposes \textbf{Ada}ptive \textbf{A}greement from \textbf{B}iased \textbf{C}ouncil, a one-stage algorithm that debiases via balancing the learning of agreement and disagreement from the guidance of a biased model.
Specifically, a biased model was trained with the general cross entropy loss which helps capture shortcuts by encouraging the model to learn easier samples.
To foster the bias learning ability, we introduced the bias council, an ensemble of classifiers learned from diversified training subsets.
To learn a debiasing model, instead of using heuristic loss weighting functions, we proposed an adaptive agreement objective by requiring the model to agree with the correct decisions and disagree with the wrong decisions made by the biased model.
Essentially, the right or wrong decisions by the biased model indicated how likely the samples were with or without spurious correlations.
Hence, learning agreement prevented the debias model from ignoring largely the rich information contained in samples with spurious correlations, and learning disagreement further drove the model to learn a different minimal via the samples without spurious correlations.
Ada-ABC could then be derived by training simultaneously the bias council and the debiasing model.
Further, we provided theoretic analysis to demonstrate that the adaptive agreement loss enforced the debiasing model to learn different features from those captured by the biased model.
To demonstrate the effectiveness of our proposed Ada-ABC on mitigating dataset biases in medical images, we carried out extensive experiments under seven different scenarios on four medical image datasets with various dataset biases.
We highlight our main contributions as follows:
\begin{itemize}
    \item We proposed Ada-ABC, a novel one-stage bias label-agnostic framework that alleviates dataset bias in medical image classification.
    \item We demonstrated theoretically that with our proposed algorithm, the debiasing model could learn the target feature when the biased model captures the bias information.
    \item To our best knowledge, we provided the first medical debiasing benchmark with four datasets under seven different scenarios covering various medical dataset biases.
    \item We validated the effectiveness of our proposed Ada-ABC in alleviating medical dataset biases under various situations based on the benchmark.
\end{itemize}

\section{Related Works}

\subsection{Dataset Bias in Medical Images}
There are many studies reported that deep learning models prefer bias information other than targeted patterns in the domain of medical image analysis.
Taking chest X-ray (CXR), the commonest medical imaging, as an example, Zech \etal discovered that CNN generalized poorly on the testing set from external sources (\ie, a different hospital).
Luo \etal \cite{luo2022rethinking} showed that classification models could learn pattern other than disease signs with quantitative analysis.
Viviano \etal \cite{viviano2020saliency} further found that CNNs could learn unwanted features outside the lungs even when restricted to learning from thoracic disease masks.
More specifically, some consistent but medically irrelevant patterns have been identified.
Chest drains, a common treatment for pneumothorax, were found to be used to identify the disease condition \cite{oakden2020hidden}.
In the study by Degrave \etal \cite{degrave2021ai}, laterality marker was found to be an evidence for the model to recognize COVID-19 signs.
Recent studies further found that imbalance of gender \cite{larrazabal2020gender}, race \cite{gichoya2022ai}, and even socioeconomic \cite{seyyed2020chexclusion} could also cause unfairness in deep learning models for under-represented groups \cite{seyyed2021underdiagnosis}.
Moreover, biases could also exist when applying deep learning models to other medical imaging domains, such as mammography \cite{zufiria2022analysis} and magnetic resonance imaging \cite{zhao2020training}.
Despite many reports of dataset biases, works on combating dataset bias in medical image classification are still scarce.
We deem that one of the main reasons is the lack of benchmarks with at least bias labels in the testing set.
In this paper, we provide a medical debiasing benchmark with four datasets under seven different scenarios with various dataset biases.

\subsection{Deep Debiased Learning}
There has been an increasing interest in developing debiased models in both the natural image domain and the medical image domain.
We here broadly categorize the related works in the following two.

\textbf{Methods using bias labels.}
A broad branch of work uses resampling or reweighting strategies to robustly learn representations for both the majority and the minority groups.
Li \etal \cite{li2019repair} proposed a minimax algorithm to automatically learning resampling weights over the training samples.
Sagawa \etal \cite{sagawa2020distributionally} proposed group distributional robust optimization (G-DRO) to prioritize the learning on worst-performing groups.
Another type of studies emphasizes learning invariant representations. 
Arjovsky \etal \cite{arjovsky2019invariant} proposed invariant risk minimization (IRM) to enforce learning invariant representations across different environment.
Zhou \etal \cite{zhou2022sparse} further impose sparsity regularization into IRM to alleviate the overfitting problem caused by overparameterization.
Similarly, contrastive learning \cite{tartaglione2021end} and mutual information minimization \cite{zhu2021learning} have also been utilized to learn more compact and invariant features across different environments.
In this paper, we study more practical situations where the biases are not explicitly labeled.
We will also show that our proposed method even achieved comparable results to the approach that used the bias labels.

\textbf{Methods without bias labels.}
Labeling biases could be tedious, and finding biases might rely on post-hoc interpretation of the model \cite{degrave2021ai}.
Efforts have also been devoted to developing debiased models without explicit bias labels.
Two-stage methods often estimate the bias distribution first and then develop debiasing model with the estimated bias information.
Sohoni \etal \cite{sohoni2020no} estimated the bias information via clustering techniques and then debiasing with G-DRO.
Liu \etal \cite{liu2021just} proposed a simple yet effective twice-training strategy that first learns an ERM model and then develops a debiasing model based on the sampling weights given by the ERM model.
Luo \etal \cite{luo2022pseudo} estimated the Bayesian distribution of the biases and target labels, and then adopted bias-balanced learning \cite{hong2021unbiased} for the second-stage debiasing training.
Nevertheless, these methods were highly sensitive to the convergence of the biased model, and wrong bias predictions could introduce much noise into the second stage.

One-stage approaches typically develop the biased and debiasing models simultaneously.
Nam \etal \cite{nam2020learning} proposed a heuristic loss weighting strategy, where a debiased model was learned by comparing its loss with another simultaneously trained biased model.
Based on this scheme, Lee \etal \cite{lee2021learning} further introduced feature augmentation by swapping the features between the two networks,
Kim \etal \cite{kim2022learning} proposed to pre-train the model first, and then use an ensemble of biased models to stabilize the learning of bias information.
The debiasing objective is then a cross entropy loss inversely weighted by the number of correct predictions given by the biased models.
However, the heuristic loss functions may over-weigh the samples without spurious correlations and insufficiently utilize the rich information contained in the majority groups.
On the contrary, our proposed Ada-ABC could assist in balancing the learning of different samples, thus outperforming other competitive algorithms.

\section{Methodology}

\subsection{Problem Setup}
Generally, a sample data $x$ could contain different attributes, \eg, whether a chest X-ray contains pneumothorax, whether the patient is male, whether a chest drain is applied, etc.
Let $(t, b)$ be the pair of (possibly latent) target and bias attributes, the values of $t$ and $b$ are binary, where 0 and 1 represent whether the attribute is absent or present, respectively.
Specifically, $t$ is used for labeling the dataset, and $b$ may not be recorded due to limited labeling budget or privacy reasons.
We can obtain a dataset $\mathcal{D} = \{(\mathit{x}_{1}, \mathit{y}_{1}), (\mathit{x}_{2}, \mathit{y}_{2}), \cdots\}$, where $y_{i}$ represents the label of $x_i$.
As the data is labeled according to the target attribute, we have $y_{i}=t_{i}$.
The dataset is biased as the bias attribute is spuriously correlated to the target attribute, \ie, $b_{i}=y_{i}$ for most of the samples.
Therefore, $b$ can be almost as predictive as $t$.
We define a sample with spurious correlation if $b_{i}=y_{i}$ or without spurious correlation if $b_{i}\neq y_{i}$.
Following previous works \cite{nam2020learning,lee2021learning,luo2022pseudo}, we strictly focus on the situations where the dataset bias is known to exist while not explicitly labeled.
Our main goal is to develop debiasing models that use the target attribute instead of the bias attribute for making decisions.

To this end, we propose a one-stage debiasing algorithm, Adaptive Agreement from Biased Council (Ada-ABC), to learn a debiased model without explicit labeling of the bias attributes.
As depicted in Fig. \ref{fig:framework}, Ada-ABC trains two networks simultaneously.
A model $f_{\theta}$ will be trained with empirical risk minimization (ERM) to learn the shortcuts as much as possible, where $f$ represents the mapping function of the model and $\theta$ represents the model's parameters.
The other model $f_{\tilde{\theta}}$ will be trained at the same time via learning adaptive agreement from $f_{\theta}$.

\begin{figure*}[t]
\begin{center}
\centering
\includegraphics[width=\linewidth]{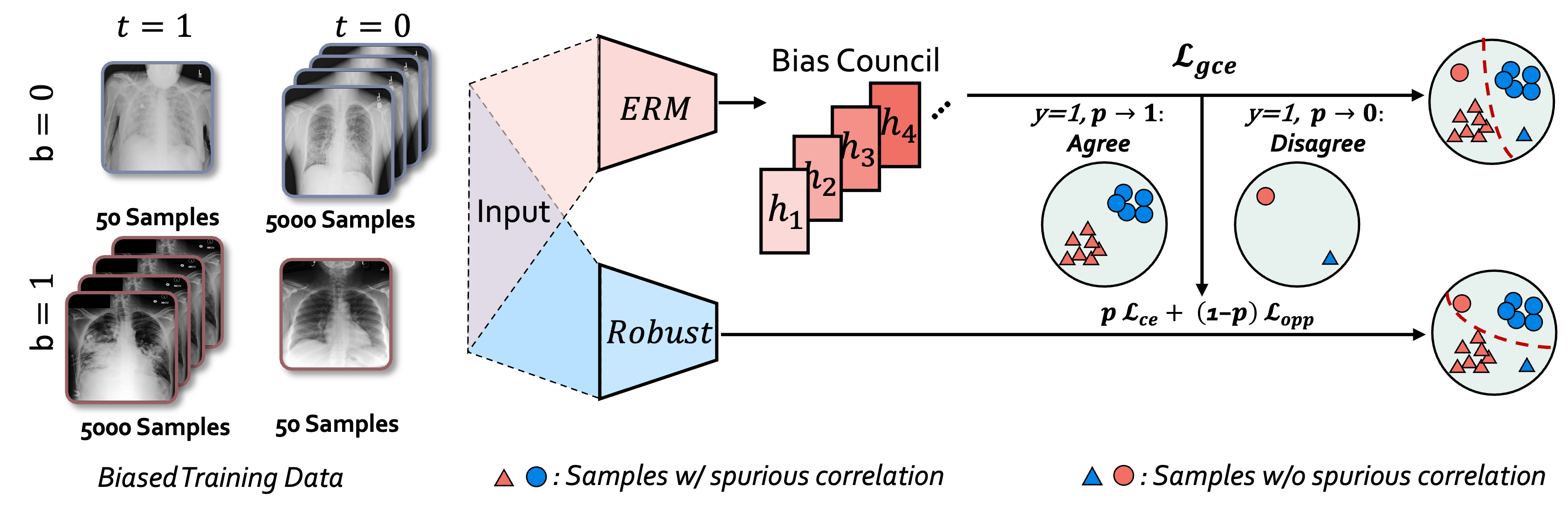}
\end{center}
   \caption{The Framework of Ada-ABC. The goal is to develop a debiasing model which is robust to dataset biases (\eg, caused by spurious correlation between the data source and health condition). A bias council with multiple classification heads is trained with empirical risk minimization, \eg, minimization of generalized cross entropy loss. A second model is simultaneously trained and required to agree with the correct predictions made by the ERM model and disagree with the wrong predictions. Under such an adaptive agreement learning scheme, a different decision-making rule can be learned from the samples w/o spurious correlations, while rich information from the samples w/ spurious correlation can be preserved as well.}
\label{fig:framework}
\end{figure*}

\subsection{Learning Adaptive Agreement}

Essentially, deep neural networks optimized by empirical risk minimization (\eg, minimization of the cross entropy loss) prioritize learning from the simple patterns \cite{arpit2017closer}.
Consequently, when an easy bias feature (\eg, chest drain) is spuriously correlated with a harder target feature (\eg, pneumothorax sign), a model optimized by empirical risk minimization prefers learning the biases \cite{nam2020learning,liu2021just,luo2022pseudo}.
Hence, samples with spurious correlation could be more correctly classified by $f_{\theta}$, while those without spurious correlation tend to be misclassified more frequently.
Our key motivation is that, if $f_{\theta}$ is well trained, it would be sufficient to learn a debiasing model by letting $f_{\tilde{\theta}}$ learn to \textit{agree} with the right decisions and \textit{disagree} with the wrong decisions made by $f_{\theta}$.

For simplicity, let $p$ and $\tilde{p}$ be the prediction by $f_{\theta,y=1}(x)$ and $f_{\tilde{\theta},y=1}(x)$, respectively.
To learn agreement, we can set $f_{\tilde{\theta}}$ to be also optimized by the cross entropy loss:
\begin{equation}
    \mathcal{L}_{\rm ce} = -\log\tilde{p}.
\label{loss:ce}
\end{equation}
As aforementioned, minimizing $\mathcal{L}_{\rm ce}$ here minimizes the empirical risk over the entire training set.
To achieve disagreement, intuitively, $f_{\tilde{\theta}}$ should make opposite predictions to $f_{\theta}$, \ie, $\tilde{p}$ should tend to 0 or 1 when $p$ approaches to 1 or 0, respectively.
Motivated by \cite{pagliardini2022agree}, we implement the following loss to drive $f_{\tilde{\theta}}$ to make an opposite prediction to $f_{\theta}$:
\begin{equation}
    \mathcal{L}_{\rm opp} = -\log(\tilde{p}(1-p)+p(1-\tilde{p})+\epsilon),
\label{loss:opposite} 
\end{equation}
where $\epsilon$ is a small value for numerical stabilization. 

While bias labels are not available, one of our main challenge here is \textit{when} to learn agreement or disagreement, which is essentially the question of when the biased model would make right or wrong decisions.
Particularly, the prediction by $f_{\theta}$ reveals whether a sample has spurious correlation, as $f_{\theta}$ is learned with ERM and can be used as an indicator for the training of the debiasing model.
A large or small value of $p$ indicates that the sample has a high potential for exhibiting spurious correlation or not.
In this way, we propose the following adaptive agreement learning loss:
\begin{equation}
\begin{split}
    \mathcal{L}_{\rm ad} 
    &= p\mathcal{L}_{ce} + (1-p)\mathcal{L}_{\rm opp} \\
    &= \mathcal{L}_{\rm agr} + \mathcal{L}_{\rm dis}.
\end{split}
\label{loss:ad}
\end{equation}

In the above, $p$ adaptively assigns different weights for the samples, where more agreement learning will be put on the samples with spurious correlations, and more disagreement will be put on the samples without spurious correlations.
In this way, $\mathcal{L}_{\rm ad}$ can be applied to all samples with $p$ adjusting the learning of agreement and disagreement.

Importantly, when the ERM model successfully captures the dataset bias $b$, it can be shown that $f_{\tilde{\theta}}$ would be driven to learn the patterns of $t$ by the following.

\begin{theorem}
{\rm (Eq. \ref{loss:ad} encourages learning the target pattern.)}
\textit{Given a joint data distribution $\mathcal{D}$ of triplets of random variables $(T, B, Y)$ taking values into $\{0,\ 1\}^3$, where $T$ represents the target feature and $B$ represents the bias feature. Assuming that an ERM model learned the posterior distribution $\mathbb P_1(Y=1|T=t,\ B=b)=b$, meaning that it is invariant to feature $t$. Then, the posterior solving $\mathcal{L}_{ad}$ objective will be $\mathbb P_2(Y=1|T=t,\ B=b)=t$, invariant to feature $b$.}
\label{Theorem:adaptive agreement learning}
\end{theorem}

\begin{proof} 
Let $T$, $B$, and $Y$ represent the random variables for the target feature, bias feature, and ground truth label.
The training set is a joint distribution $\mathcal D$ of triplets of $(T,\ B,\ Y)$ taking values in $\{0,\ 1\}^3$.
For simplicity, we further let $\mathop{\mathcal D}_{t=b}$ and $\mathop{\mathcal D}_{t\neq b}$ to be uniform on $\{T, B\}$, but the following still holds if the distribution is not uniform.
In other words, 
\begin{equation}
\begin{split}
            &\mathbb{P}_{\mathcal{D}}(t=0,b=1|t\neq b)=\mathbb{P}_{\mathcal{D}}(t=1,b=0|t\neq b)=1/2, \\
        & \mathbb{P}_{\mathcal{D}}(t=1,b=1|t=b)=\mathbb{P}_{\mathcal{D}}(t=0,b=0|t=b)=1/2.
\end{split}
\end{equation}

Let $\mathbb{P}$ and $\mathbb{\tilde{P}}$ be the learned distribution of the biased model and debiasing model, respectively.
We further assume that the biased model learned the posterior distribution $\mathbb{P}(Y=1|T=t,\ B=b)=b$.

Training the debiasing model would minimize the expectation of $\mathcal L_{ad}$ objective over $\mathcal D$:

\begin{equation}
\begin{split}
    \min &\mathop{\mathbb{E}}\limits_{(t,\ b)\sim \mathcal D}\left[ \mathcal{L}_{\rm agr} + \mathcal{L}_{\rm dis} \right]\\
    &= \mathop{\mathbb{E}}_{(t,\ b)\sim \mathcal D} \left[ p\mathcal L_{ce}  \right]+\mathop{\mathbb{E}}_{(t,\ b)\sim \mathcal D}\left[(1-p)\mathcal{L}_{\rm opp} \right].
\end{split}
\label{formulae: minex}
\end{equation}

By the mentioned conditions,
the biased model successfully captures the bias.
Eq. \ref{formulae: minex} can then be further re-written to:

\begin{equation}
        \min \mathop{\mathbb{E}}_{(t,\ b)\sim \mathcal D\atop t=b}\left[ 
\mathcal L_{ce} \right] + \mathop{\mathbb{E}}_{(t,\ b)\sim \mathcal D\atop t\neq b}\left[\mathcal{L}_{\rm opp} \right],
\label{formula:minex_inference}
\end{equation}

\noindent where the first term in Eq. \ref{formula:minex_inference} is minimized for agreement in the distribution of bias-aligned data, which is the empirical risk minimization over $\mathop{\mathcal{D}}\limits_{t=b}$:
\begin{equation}
\begin{cases}
    &\mathbb{\tilde{P}}(Y=1|t=1,\ b=1)=1,\\
    &\mathbb{\tilde{P}}(Y=1|t=0,\ b=0)=0.
\end{cases}
\label{formula:agr_exp}
\end{equation}
The second term in Eq. \ref{formula:minex_inference} becomes:
\begin{equation}
\begin{split}
 &\mathop{\mathbb{E}}\limits_{(t,\ b)\sim \mathcal D\atop t\neq b}\left[-\log (\Tilde{p}(1-p)+p(1-\Tilde{p})) \right]=\\
&\frac{1}{2}\left[ -\log \left( 1-\mathbb{\tilde{P}}(Y=1|t=0,\ b=1) \right) \right] \\
&+ \frac{1}{2}\left[ -\log \left( \mathbb{\tilde{P}}(Y=1|t=1,\ b=0) \right) \right]   
\end{split}
\end{equation}  

\noindent which is minimized for 
\begin{equation}
\begin{cases}
    &\mathbb{\tilde{P}}(Y=1|t=0,\ b=1)=0,\\
    &\mathbb{\tilde{P}}(Y=1|t=1,\ b=0)=1.
\end{cases}
\label{formula:dis_exp}
\end{equation}

\noindent Combining Eq. \ref{formula:agr_exp} and Eq. \ref{formula:dis_exp}, the posterior learned by the debiasing model according to our proposed adaptive agreement learning loss will be:
\begin{equation}
    \mathbb{\tilde{P}}(Y=1|T=t,\ B=b)=t,
\end{equation}
\noindent which shows that the debiasing model learns the targeted feature invariant to the bias.
\end{proof}

\subsection{Learning the Bias Council}
By Eq. \ref{loss:ad}, the samples on which $f_{\tilde{\theta}}$ should learn agreement or disagreement are decided by $p$.
Then, the following challenge is how to make $p$ successfully indicate the dataset bias information.
In other words, the next goal is to make $f_{\theta}$ as biased as possible.

To this end, we presented the combination of the generalized cross entropy (GCE) loss \cite{zhang2018generalized} and a diversely trained ensemble of classifiers.
The GCE loss was first proposed for learning from noisy labels, where the samples with clean labels are often regarded easier samples than those with noisy labels:
\begin{equation}
    \mathcal{L}_{\rm gce} = \frac{1-f_{\theta}(x)^{q}}{q},
\label{loss:gce}
\end{equation}
where $q\in(0,1]$ is a hyper-parameter balancing the behavior of the loss.
For each class (\ie, the target label is 1), $\mathcal{L}_{\rm gce}$ generates to the mean absolute error loss when $q=1$ and behaves like conventional cross entropy loss when $q\to 1$. This can be seen by its gradient: 
\begin{equation}
\begin{split}    
    \frac{\partial \mathcal{L}_{\rm gce}(f_{\theta}(x))}{\partial \theta} 
    & = -{f_{\theta}(x)}^{q-1}\frac{\partial f_{\theta}(x)}{\partial \theta}\\
    & = {f_{\theta}(x)}^{q}(-f_{\theta}(x)^{-1}\frac{\partial f_{\theta}(x)}{\partial \theta})\\
    & = {f_{\theta}(x)}^{q}\frac{\partial \mathcal{L}_{\rm ce}(f_{\theta}(x))}{\partial \theta},
\end{split}
\end{equation}
where $\mathcal{L}_{\rm ce}(f_{\theta}(x)) = -{\rm log}(f_{\theta}(x))$.
The above shows that the gradient of the GCE loss is a weighted version of the gradient of the cross entropy loss, and the weight is given by its own prediction.
In other words, this loss encourages the model to be confident in its prediction by up-weighing the samples with high predicted probabilities and down-weighing the samples otherwise.
As mentioned, the samples with higher probabilities are highly potentially the easy samples with spurious correlations.
Hence, $f_{\theta}$ would focus on learning from these samples to capture the dataset bias.

To further facilitate robust bias learning, we proposed to train a biased council, which consists of an ensemble of diverse GCE-optimized classifiers. 
Specifically, we introduced a group of classification heads $\{h^{i}_{\phi_{i}}\}_{1}^{n}$ to $f_{\theta}$, where $\phi_{i}$ represents the parameters of $h^{i}$.
The $i$-th head would be trained with a subset $\mathcal{D'}_{i}$ randomly sampled from $\mathcal{D}$.
The parameters were also randomly and independently initialized for each head.
Thus, the increased diversity in the training set and parameters would promote the classifiers to learn a more robust ensemble \cite{nam2021diversity}.
Finally, the prediction by $f$ is set to be the average of the head predictions, \ie, $p=f_{\theta}(x)=\sum_{1}^{n}h^{i}_{\phi_{i}}(z)$, where $z$ is the feature fed to the classifiers.

\subsection{Holistic Training of Ada-ABC}
Combining the above, the training objective for learning adaptive agreement from bias council can be derived as:
\begin{equation}
\mathop{\arg\min}_{\tilde{\theta},\theta}
    \mathcal{L}_{\rm agr}(\tilde{\theta})+\lambda\mathcal{L}_{\rm dis}(\tilde{\theta})+\mathcal{L}_{gce}(\theta),
\label{loss:overall}
\end{equation}
where we add a hyper-parameter $\lambda$ to help balance the learning of agreement and disagreement in case the samples are too imbalanced.
Note that $\mathcal{L}_{\rm agr}$ and $\mathcal{L}_{\rm dis}$ will not be back-propagated to $f_{\theta}$ to avoid influence from the gradient of $f_{\tilde{\theta}}$.
The above loss terms can be optimized simultaneously and need not the convergence of a biased model as a first step.

Algorithm \ref{Algorithm:Ada-ABC Training} details the one-stage training process.
We emphasize that Ada-ABC does not require any explicit bias labels and only requires the knowledge that a training set is biased.
Furthermore, with the convergence of the biased model, the proposed adaptive learning scheme enables the debiasing model to learn a different feature from the wrongly predicted samples and also keeps the rich knowledge from sufficient samples with spurious correlations.

\begin{algorithm}[ht]
    \caption{Ada-ABC Training}\label{Algorithm:Ada-ABC Training}
    
    \hspace*{\algorithmicindent} \textbf{Input:} Dataset $\mathcal{D}=\{\mathcal{X}, \mathcal{Y}\}$; Parameters of biased model $\theta$; Parameters of debiased model $\tilde{\theta}$; Hyper-parameter $\lambda$.\\    
    \hspace*{\algorithmicindent} 
    \textbf{Output:} Debiased model $f_{\tilde{\theta}}$.
    \begin{algorithmic}[1]
    \State{Initialize parameters $\theta$, and $\tilde{\theta}$.}
    \While{not converge} 
        \State{$(X,Y) \sim \mathcal{D}$}
        \State{$P \xleftarrow{} f_{\theta}(X)$}
        \State{$\ell_{\rm agr} \xleftarrow{} \mathcal{L}_{\rm agr}(\tilde{\theta};X,Y,P)$} \hfill $\triangleright$ (\ref{loss:ce}),(\ref{loss:ad})
        \State{$\ell_{\rm dis} \xleftarrow{} \mathcal{L}_{\rm dis}(\tilde{\theta};X,Y,P)$} \hfill $\triangleright$ (\ref{loss:opposite}),(\ref{loss:ad}) 
        \State{$\ell_{\rm ad} \xleftarrow{} \ell_{\rm agr} + \lambda\ell_{\rm dis}$}  \hfill$\triangleright$ (\ref{loss:ad})
        
        \State{$\tilde{\theta} \xleftarrow{} \tilde{\theta}-\eta\nabla_{\tilde{\theta}}\ell_{\rm ad}$}
        
        \State{$\ell_{\rm gce} \xleftarrow{} \mathcal{L}_{\rm gce}(\theta;X,Y)$}
        \hfill $\triangleright$ (\ref{loss:gce})
        
        \State{$\theta \xleftarrow{} \theta-\eta\nabla_{\theta}\ell_{\rm gce}$}

    \EndWhile

    \end{algorithmic}
\end{algorithm}

\section{Experiments}
\label{experiments}

\subsection{Medical Debiasing Benchmark}
\label{sec:MDB}
We constructed the first medical debiasing benchmark (MBD) with four real-world medical image datasets, including Source-biased Pneumonia classification dataset (SbP), Gender-biased Pneumothorax classification dataset (GbP), Chest Drain-biased Pneumothorax classification dataset (DbP), and the OL3I Dataset, with totally seven different dataset bias scenarios.
Detailed numbers of data used for each dataset can be found in Table \ref{tab:datasets}. In the following, we use $t$ and $b$ to represent the target and bias attributes, respectively.

\subsubsection{Source-biased Pneumonia Classification (SbP)}
Chest X-rays (CXRs) generated from different clinical centers could have distribution shifts caused by factors such as imaging parameters, vendor types, patient cohort differences, etc. \cite{luo2020deep}.
SbP is a pneumonia classification dataset containing most pneumonia cases from MIMIC-CXR \cite{johnson2019mimic} and most healthy cases (no findings) from NIH-CXR \cite{wang2017chestx}.
Here, $t=health\_condition$, and $b=data\_source$.
Further, there are three training sets with the ratios of bias-aligned samples of 99\%, 95\%, and 90\%, respectively.
The three scenarios share the same validation and testing sets, which have uniform distributions on the $t$ and $b$ variables, i.e., containing equal numbers of different groups (with pneumonia or without pneumonia; from NIH or MIMIC-CXR).

\subsubsection{Gender-biased Pneumothorax Classification (GbP)}
Significant performance decreases of a pneumothorax classifier have been witnessed when training with male cases and testing on female cases (and vice versa) \cite{larrazabal2020gender}.
GbP dataset was collected from the NIH-CXR dataset, where $t=health\_condition$ and $b=gender$.
It contains two training sets with most male patients (case 1) and most female patients (case 2), respectively.
The validation and testing sets are with uniform distributions of $t$ and $b$.

\subsubsection{Chest Drain-biased Pneumothorax Classification (DbP)}
Chest drain is a common treatment for pneumothorax and have been reported as a type of dataset biases \cite{rueckel2020impact,oakden2020hidden}.
DbP was collected from the NIH-CXR dataset, and the chest drain labels were provided by \cite{oakden2020hidden}.
In the training set, most pneumothorax cases contain chest drains, and all healthy cases do not contain chest drains.
Here, $t=health\_condition$, and $b=chest\_drain$.

\subsubsection{Age-biased Ischemic Heart Disease Prognosis (OL3I)}

The Opportunistic L3 Ischemic heart disease (OL3I) dataset \cite{zambrano2023opportunistic} provided abdominopelvic computed tomography images at the third lumbar vertebrae (L3) level for opportunistic assessment of ischemic heart disease risk.
In this paper, we predicted the ischemic heart disease risk one year after the examination.
According to \cite{zong2023medfair}, age is spuriously correlated to the one-year risk, where individuals with age larger than 60 is less likely to be healthy and more likely to obtain ischemic heart disease within one year.
In other words, $t=ischemic\_heart\_disease\_risk$, and $b=age$.
We followed the original split of the OL3I dataset.

\begin{table}
\centering
\caption{Detailed number of data in different datasets. $t$ and $b$ represent the target and bias attribute, respectively. The meaning of $t$ and $b$ for each dataset can be found in Sec. \ref{sec:MDB}.}
\resizebox{\linewidth}{!}{ 
\begin{tblr}{
  cells = {c},
  cell{1}{3} = {c=2}{},
  cell{1}{5} = {c=2}{},
  cell{1}{7} = {c=2}{},
  cell{3}{1} = {r=2}{},
  cell{5}{1} = {r=2}{},
  cell{7}{1} = {r=2}{},
  cell{10}{1} = {r=2}{},
  cell{12}{1} = {r=2}{},
  hline{1-2,9,14,17,20} = {-}{},
}
Dataset              & $t$          & Training   &           & Validation &           & Testing    &           \\
SbP                     &                & $b=0$      & $b=1$       & $b=0$      & $b=1$       & $b=0$      & $b=1$       \\
 ($\rho$=99\%) & $1$      & 5,000      & 50        & 200        & 200       & 400        & 400       \\
                     & $0$        & 50         & 5,000     & 200        & 200       & 400        & 400       \\
 ($\rho$=95\%) & $1$      & 5,000      & 250       & 200        & 200       & 400        & 400       \\
                     & $0$        & 250        & 5,000     & 200        & 200       & 400        & 400       \\
 ($\rho$=90\%) & $1$      & 5,000      & 500       & 200        & 200       & 400        & 400       \\
                     & $0$        & 500        & 5,000     & 200        & 200       & 400        & 400       \\
GbP                     &                & $b=0$       & $b=1$    & $b=0$       & $b=1$    & $b=0$       & $b=1$    \\
(Case1)        & $1$   & 800        & 100       & 150        & 150       & 250        & 250       \\
                     & $0$        & 100        & 800       & 150        & 150       & 250        & 250       \\
(Case2)        & $1$   & 100        & 800       & 150        & 150       & 250        & 250       \\
                     & $0$        & 800        & 100       & 150        & 150       & 250        & 250       \\
DbP                     &                & $b=0$   & $b=1$ & $b=0$   & $b=1$ & $b=0$   & $b=1$ \\
              & $1$   & 500        & 50        & 50         & 50        & 100        & 100       \\
                     & $0$        & 0          & 1,000     & 0          & 200       & 0          & 400       \\
OL3I                     &                & $b=0$      & $b=1$       & $b=0$      & $b=1$       & $b=0$      & $b=1$       \\
             & $1$  & 87         & 141       & 13         & 43        & 25         & 141       \\
                     & $0$        & 3,512      & 1,487     & 830        & 417       & 1,060      & 478       \\ 
\end{tblr}
}
\label{tab:datasets}
\end{table}

\subsubsection{Evaluation Metrics}
\label{sec:datasets_and_metrics}

\begin{figure*}[t]
\centering
    \begin{subfigure}{.3\textwidth}
    \centering
      \includegraphics[height=3.2cm]{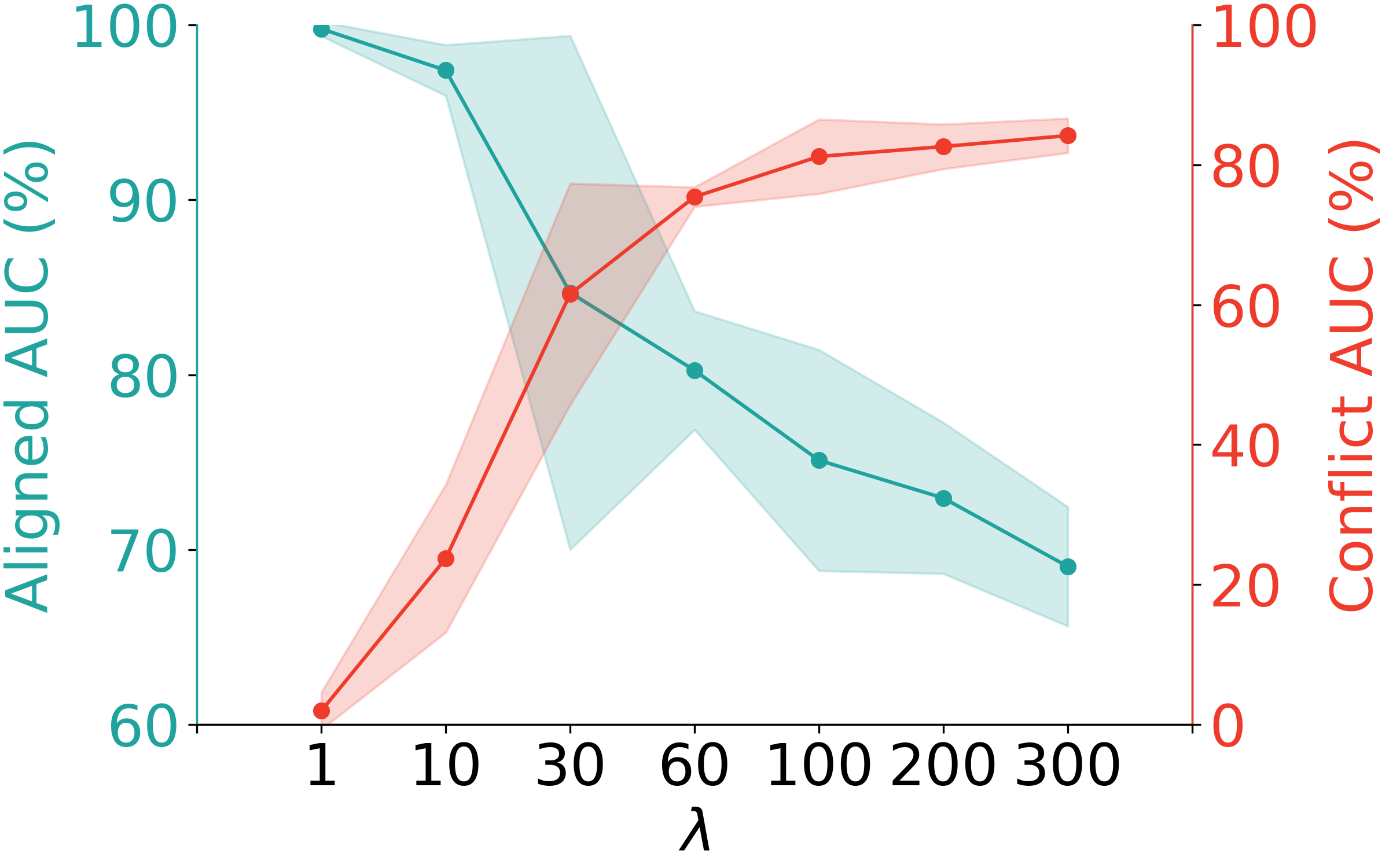}
    \caption{}
    \label{fig:lambda_SbP99_align_conflit}
    \end{subfigure}
    \begin{subfigure}{.34\textwidth}
    \centering
      \includegraphics[height=3.2cm]{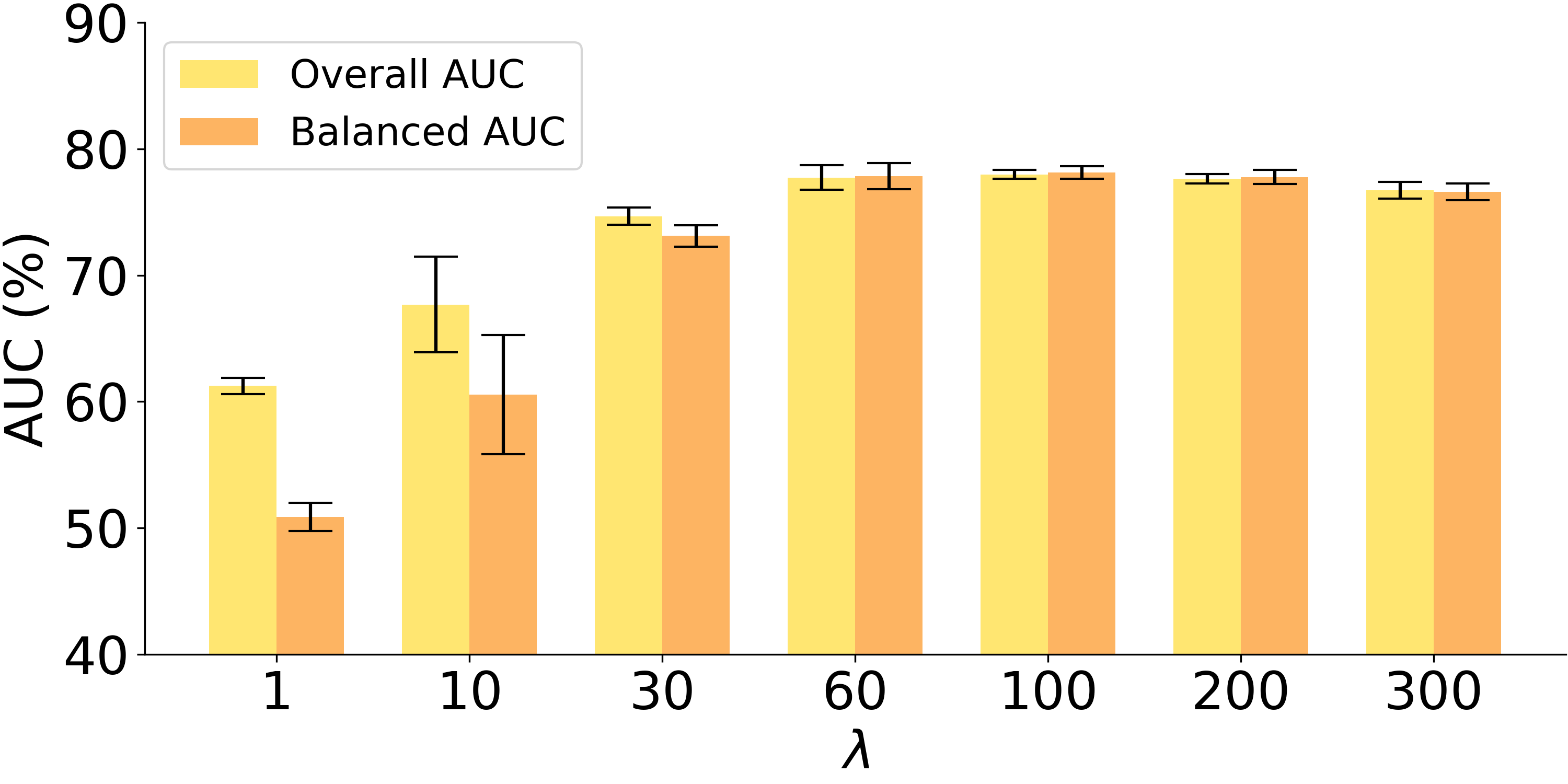}
    \caption{}
    \label{fig:lambda_SbP99_overall_balanced}
    \end{subfigure}
    \begin{subfigure}{.34\textwidth}
    \centering
      \includegraphics[height=3.2cm]{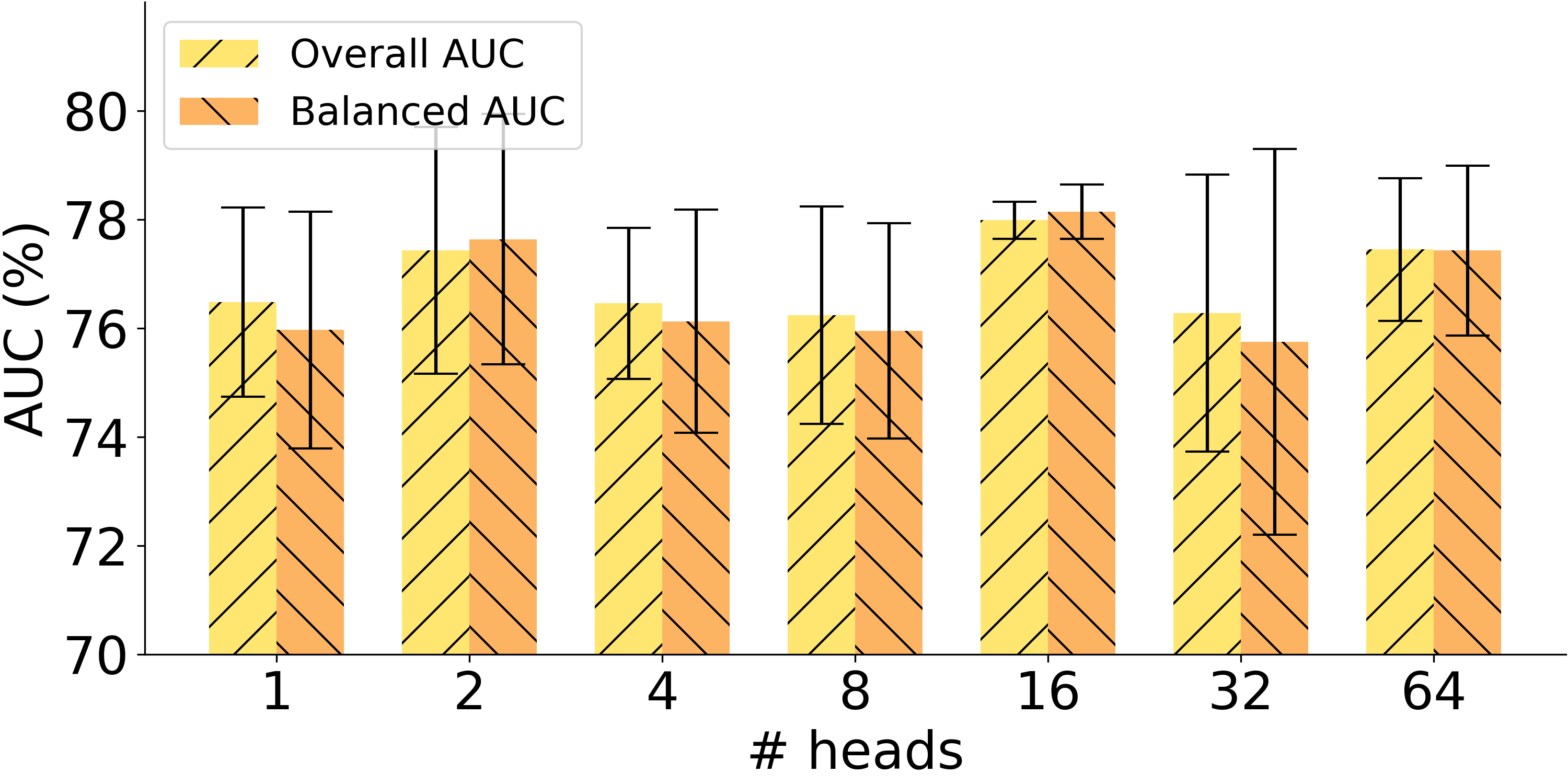}
    \caption{}
    \label{fig:head_SbP99_overall_balanced}
    \end{subfigure}
\\\vspace{1mm}

\centering
    \begin{subfigure}{.3\textwidth}
    \centering
      \includegraphics[height=3.2cm]{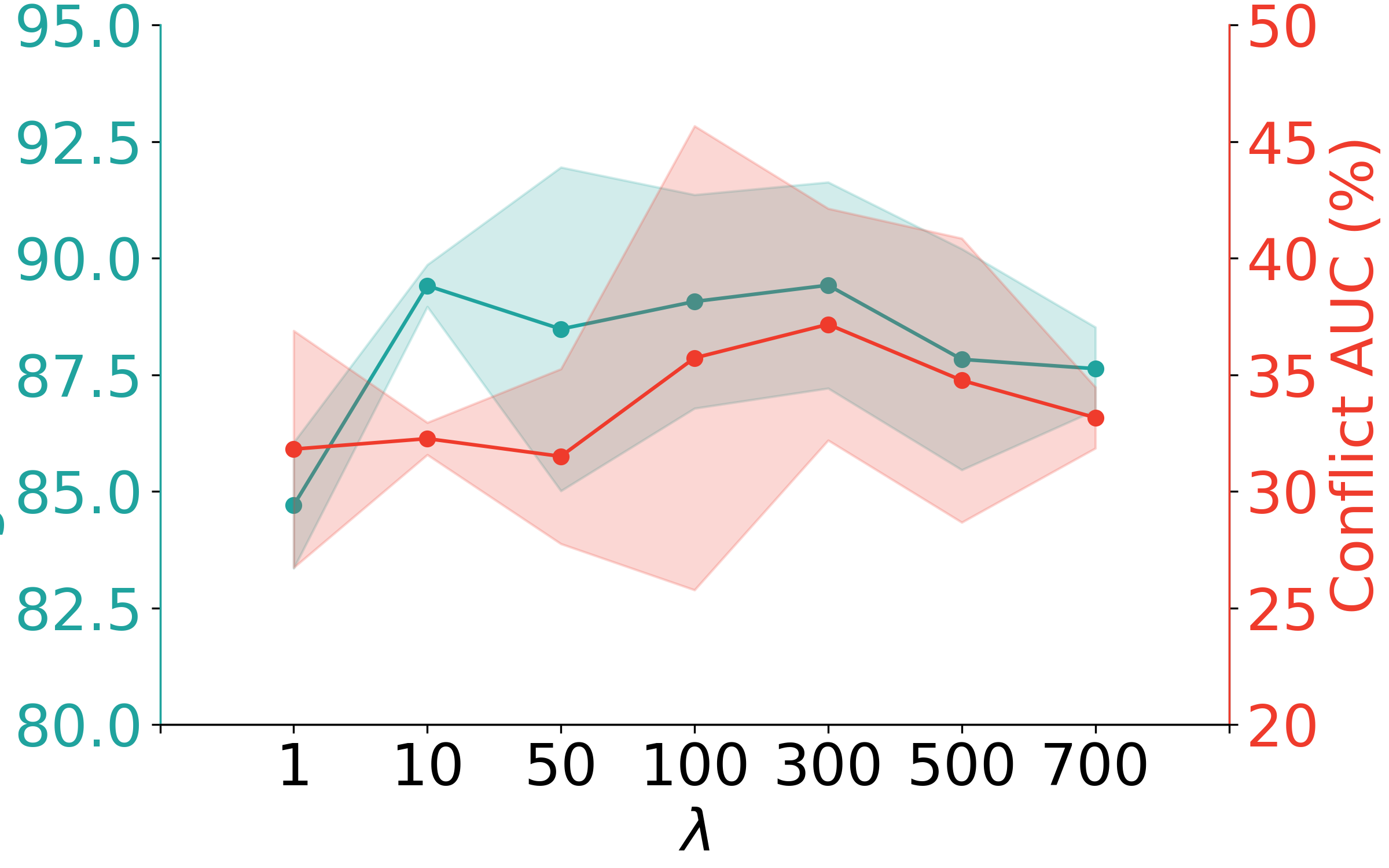}
    \caption{}
    \label{fig:lambda_Ol3I_align_conflit}
    \end{subfigure}
    \begin{subfigure}{.34\textwidth}
    \centering
      \includegraphics[height=3.2cm]{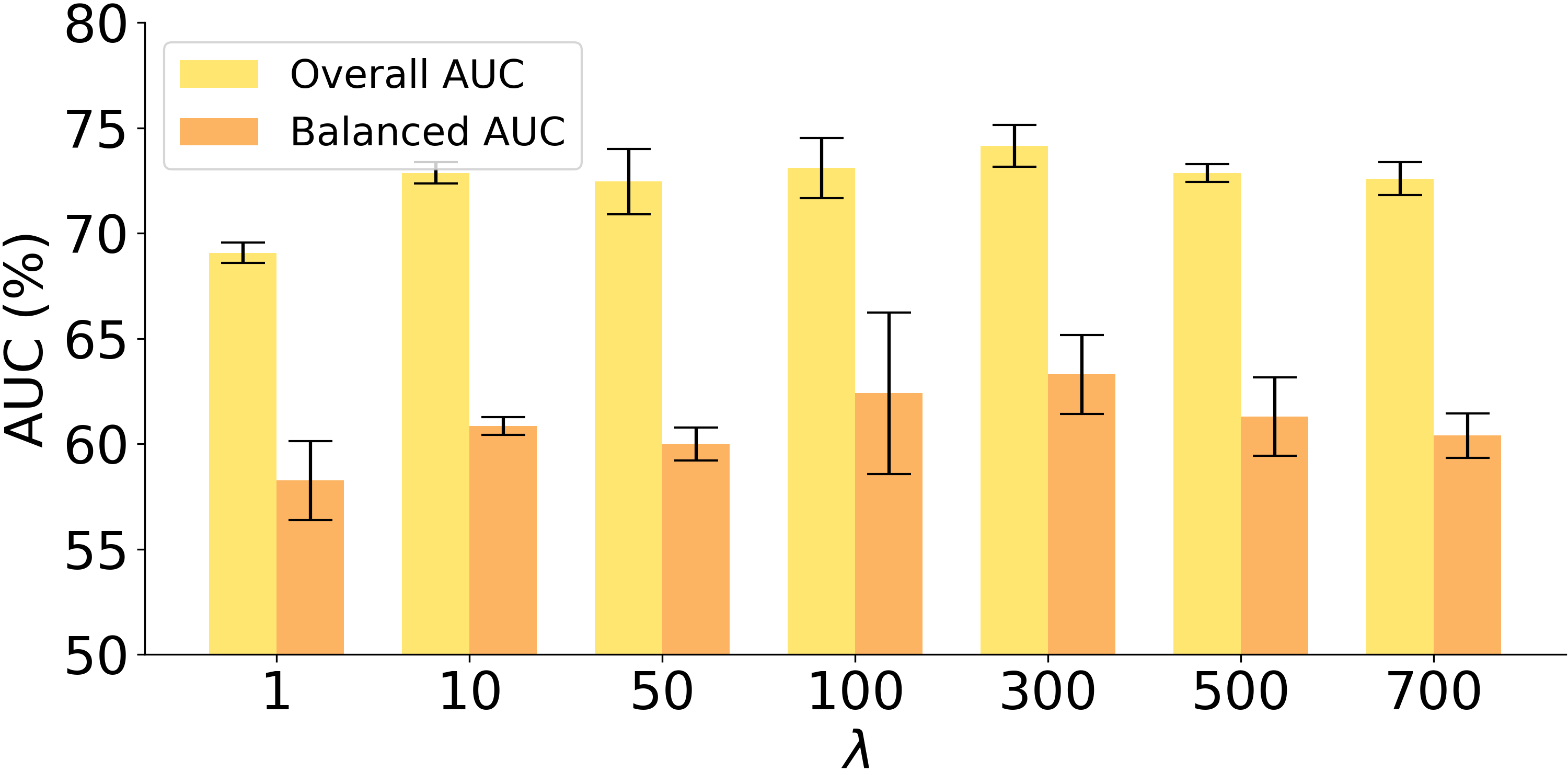}
    \caption{}
    \label{fig:lambda_Ol3I_overall_balanced}
    \end{subfigure}
    \begin{subfigure}{.34\textwidth}
    \centering
      \includegraphics[height=3.2cm]{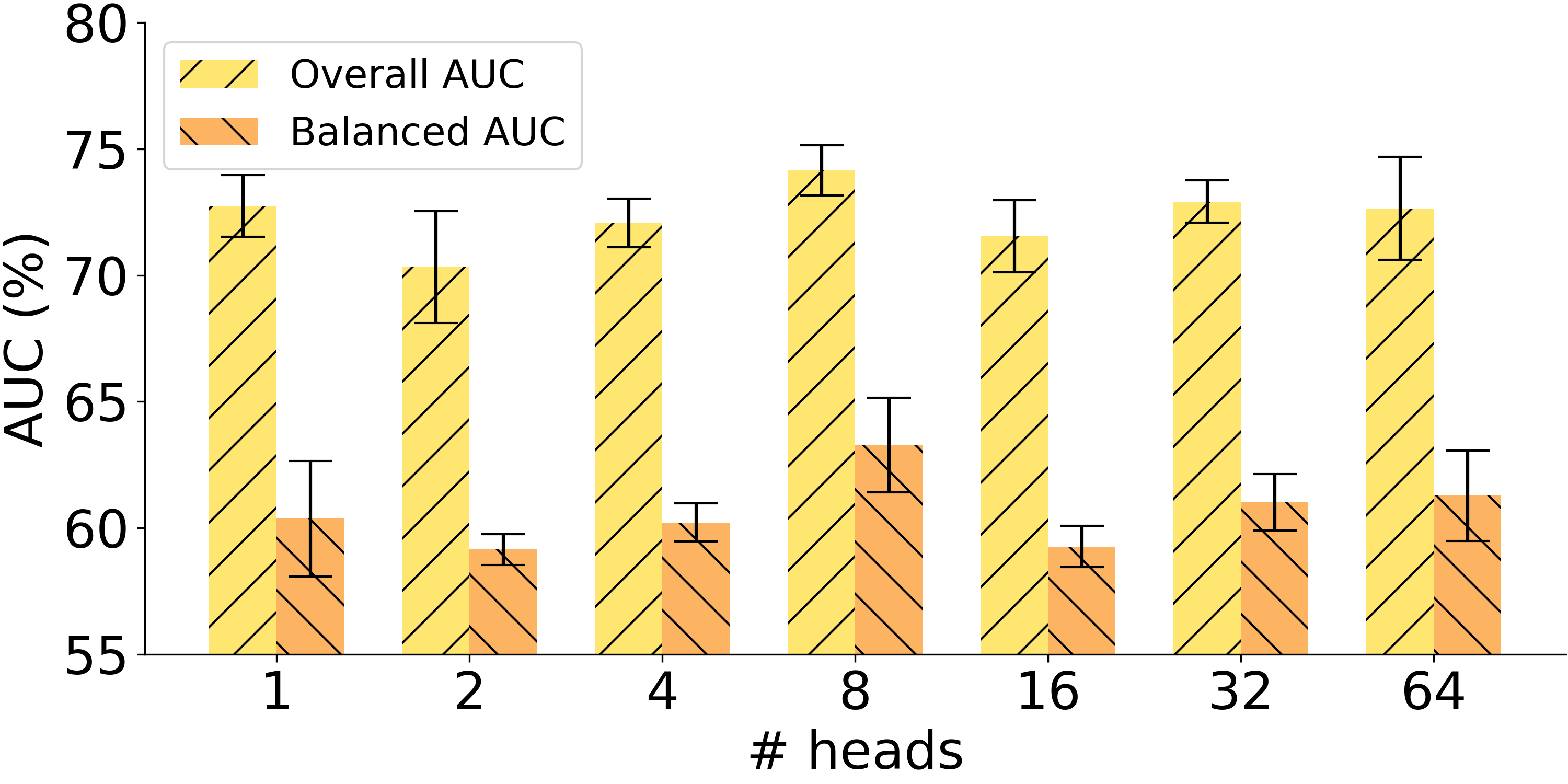}
    \caption{}
    \label{fig:head_Ol3I_overall_balanced}
    \end{subfigure}
    
\caption{Effects of the hyper-parameters $\lambda$ and number of heads. The first row shows the results on SbP dataset with $\rho=99\%$: (a) The changes of aligned AUC and conflicting AUC w.r.t. the change of $\lambda$ (\# heads = 16). (b) The changes of overall AUC and balanced AUC w.r.t. the change of $\lambda$ (\# heads = 16). (c) The changes of overall AUC and balanced AUC w.r.t. the change of number of heads ($\lambda$ = 100). The second row shows the results on OL3I dataset: (d) The changes of aligned AUC and conflicting AUC w.r.t. the change of $\lambda$ (\# heads = 8). (e) The changes of overall AUC and balanced AUC w.r.t. the change of $\lambda$ (\# heads = 8). (f) The changes of overall AUC and balanced AUC w.r.t. the change of number of heads ($\lambda$ = 300).}
\label{fig:Ablation}
\end{figure*}

In the testing phase, following \cite{nam2020learning,lee2021learning,luo2022pseudo}, a sample is called bias-aligned if its attributes are spuriously correlated in the training data, e.g., pneumothorax cases with chest drains.
A sample is called bias-conflicting if it has attributes contradict to the bias-aligned samples, e.g., pneumothorax cases without chest drains.
To observe the debiasing results on different samples, we compute four types of area under the receive operating characteristic curve (AUC):
i) \textbf{bias-aligned AUC}, which is the AUC computed on bias-aligned samples; ii) \textbf{bias-conflicting AUC} computed on the bias-conflicting samples; iii) \textbf{balanced AUC}, which is the average of bias-aligned AUC and bias-conflicting AUC; and iv) \textbf{overall AUC} computed on all samples.

The bias-aligned AUC and the bias-conflicting AUC are mainly used as a reference to tell whether the model is highly biased.
For example, when the bias-aligned AUC is too higher than the bias-conflicting AUC, the model could correctly classify samples with $t=b$ but not sample with $t\neq b$, which means it's highly biased.
As it's important to correctly classify all groups of data, the balanced AUC and overall AUC are used as a fair evaluation for different models.

\subsection{Analysis of Ada-ABC}
\subsubsection{Analysis with a Toy Example}
We first evaluate the effects of different learning schemes with a toy example, where the model is optimized to distinguish the samples ($\triangle$ vs. $\square$) according to their coordinates.
Fig. \ref{ERM} shows the decision boundary of a vanilla model optimized by empirical risk minimization, where the minority groups of samples are misclassified.
By learning purely agreement, \ie, $\mathcal{L}_{\rm agr}$, a second model could learn a similar decision boundary, as shown in Fig. \ref{Agree}.
Notably, learning purely disagreement, \ie, $\mathcal{L}_{\rm dis}$, a second model could generate the exactly opposite decisions to the vanilla model.
Finally, we trained a debiasing model using our proposed adaptive agreement learning in Eq. \ref{loss:ad} with $\lambda$ set to 1, and a correct decision boundary could be achieved as shown in Fig. \ref{AD}.

\begin{figure}[t]
\centering
\begin{subfigure}{.11\textwidth}
  \centering
  \includegraphics[width=\linewidth]{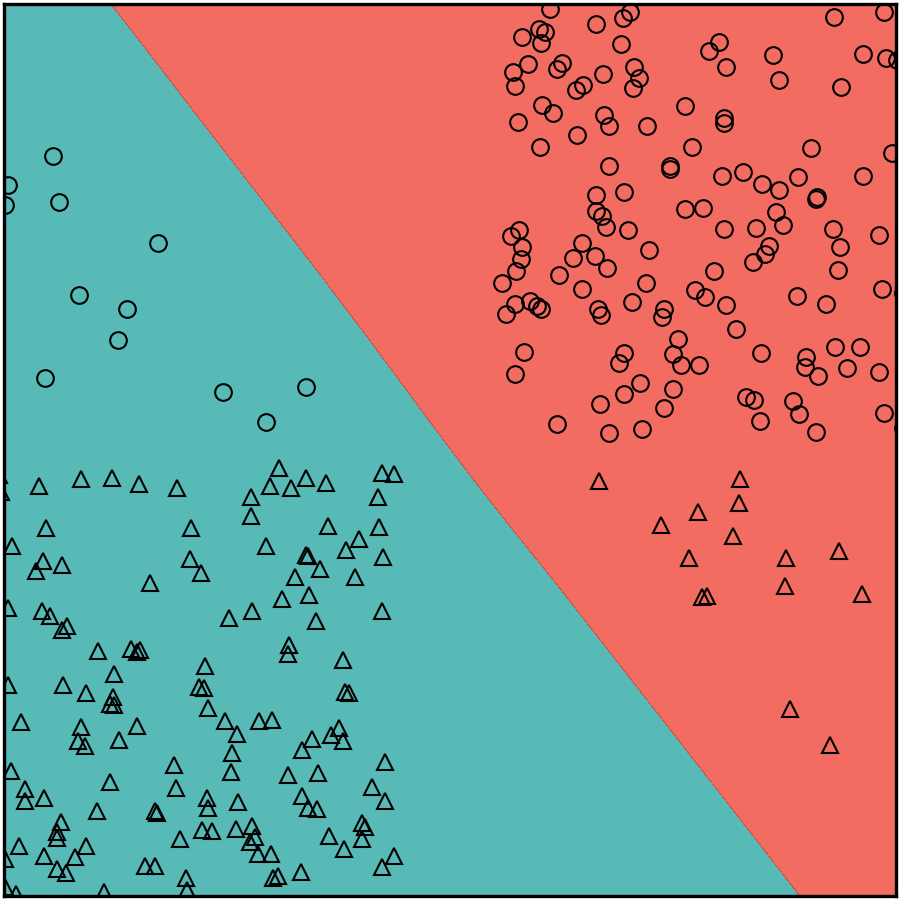}
  \caption{}
  \label{ERM}
\end{subfigure}
\begin{subfigure}{.11\textwidth}
  \centering
  \includegraphics[width=\linewidth]{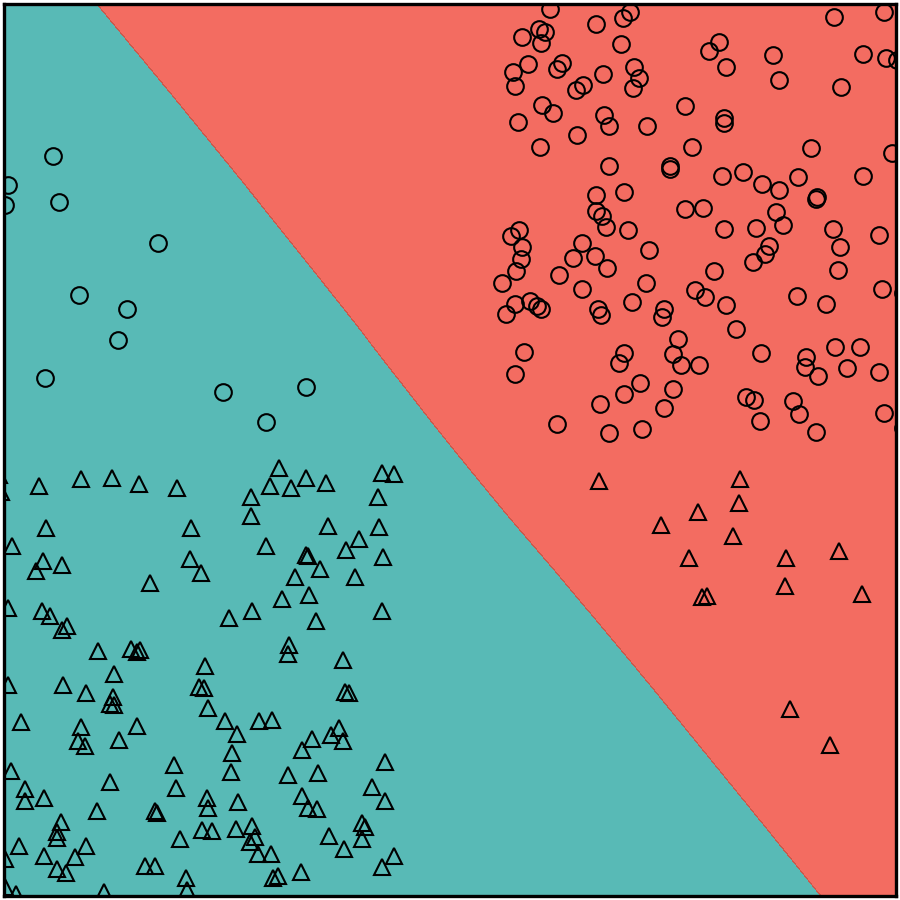}
  \caption{}
  \label{Agree}
\end{subfigure}
\begin{subfigure}{.11\textwidth}
  \centering
  \includegraphics[width=\linewidth]{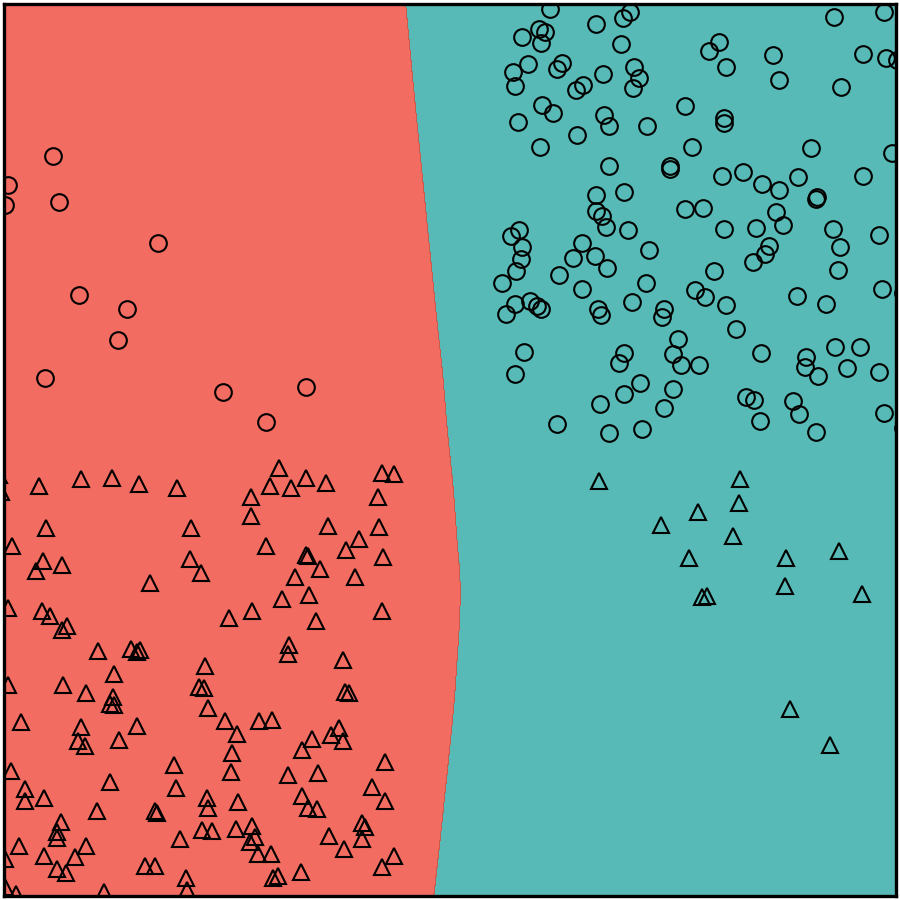}
  \caption{}
  \label{Disagree}
\end{subfigure}
\begin{subfigure}{.11\textwidth}
  \centering
  \includegraphics[width=\linewidth]{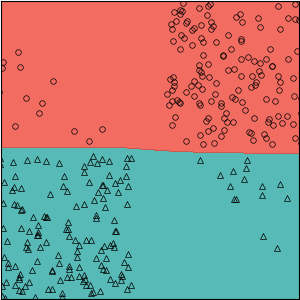}
  \caption{}
  \label{AD}
\end{subfigure}

\caption{The decision boundaries by (a) an ERM model that learns a simple solution; another model that learns to (b) purely agree with the ERM model, or (c) purely disagree with the ERM model, (d) or adaptively agree or disagree with the ERM model. Details are best appreciated when enlarged.}
\label{toy}
\end{figure}

\begin{table}[t]
\centering
\begin{tabular}{cccc}
\hline
Dataset & $\lambda$ & \# heads \\
\hline
SbP bias90 & 5 & 2 \\
SbP bias95 & 10 & 64 \\
SbP bias99 & 100 & 16 \\
GbP case1 & 1 & 128 \\
GbP case2 & 0.1 & 64 \\
DbP & 1 & 4 \\
OL3I & 300 & 8 \\
\hline
\end{tabular}
\caption{Summary of hyper-parameters for each dataset.}
\label{tab:hyper-parameters}
\end{table}

\begin{table*}[t]
\setlength{\tabcolsep}{5pt}
\centering
\caption[Comparison on the SbP dataset]{Comparison on SbP. For methods do not use bias labels, the best and second-best performance (on the balanced AUC and overall AUC) are in \best{red} and \secondbest{blue}, respectively. The last row is the average of the mean overall AUC from all seven scenarios. $\rho$: the ratio of bias-aligned samples in the training set. $\dagger$ means that the method uses ground truth bias labels.}
\resizebox{\textwidth}{!}{
\begin{tabular}{cccccccccc}
\toprule[1.5pt]
\textbf{Dataset} & \textbf{Metric} & \textbf{G-DRO$^{\dagger}$ \cite{sagawa2020distributionally}} & \textbf{ERM} & \textbf{D-BAT \cite{pagliardini2022agree}} & \textbf{JTT \cite{liu2021just}} & \textbf{PBBL \cite{luo2022pseudo}} & \textbf{LfF \cite{nam2020learning}} & \textbf{DFA \cite{lee2021learning}}& \textbf{Ada-ABC} \\ 
\hline
\multirow{4}{*}{SbP ($\rho=99\%$)} & Aligned & 74.30$_{\pm2.28}$ & 99.03$_{\pm0.95}$ & 45.40$_{\pm16.17}$ & 97.02$_{\pm1.07}$ & 72.40$_{\pm0.71}$ & 77.50$_{\pm11.08}$ & 69.33$_{\pm1.74}$ & 75.11$_{\pm6.32}$\\
& Conflicting & 85.18$_{\pm1.26}$ & 4.93$_{\pm3.68}$ & 73.18$_{\pm14.68}$ & 19.54$_{\pm4.33}$ & 77.61$_{\pm0.45}$ & 64.38$_{\pm8.75}$ & 75.48$_{\pm2.61}$ & 81.20$_{\pm5.32}$\\
& Balanced & 79.74$_{\pm0.55}$ & 51.98$_{\pm1.60}$ & 59.29$_{\pm2.68}$ & 58.28$_{\pm1.87}$ & \secondbest{75.00$_{\pm0.18}$} & 70.94$_{\pm1.30}$ & 72.40$_{\pm0.48}$ & \best{78.15$_{\pm0.50}$}\\
\rowcolor[gray]{0.9}& Overall & 79.71$_{\pm0.40}$ & 59.21$_{\pm3.76}$ & 61.37$_{\pm0.92}$ & 64.37$_{\pm1.45}$ & \secondbest{74.70$_{\pm0.14}$} & 71.86$_{\pm1.72}$ & 72.49$_{\pm0.45}$ &\best{77.99$_{\pm0.34}$}\\

\multirow{4}{*}{SbP ($\rho=95\%$)} & Aligned & 68.65$_{\pm1.21}$ & 97.91$_{\pm0.75}$ & 64.55$_{\pm15.39}$ & 92.09$_{\pm3.86}$ & 71.72$_{\pm6.65}$ & 69.56$_{\pm2.01}$ & 69.04$_{\pm4.21}$ & 84.70$_{\pm1.66}$\\
& Conflicting & 89.86$_{\pm0.67}$ & 20.45$_{\pm5.96}$ & 67.82$_{\pm16.60}$ & 45.75$_{\pm11.94}$ & 84.68$_{\pm3.49}$ & 86.43$_{\pm1.67}$ & 84.94$_{\pm2.56}$ & 73.64$_{\pm2.22}$\\
& Balanced & 79.26$_{\pm0.47}$ & 59.18$_{\pm2.61}$ & 66.18$_{\pm1.75}$ & 68.92$_{\pm4.11}$ & \secondbest{78.20$_{\pm0.20}$} & 77.99$_{\pm0.18}$ & 76.99$_{\pm0.85}$ & \best{79.17$_{\pm0.47}$}\\
\rowcolor[gray]{0.9}& Overall & 79.80$_{\pm0.36}$ & 67.11$_{\pm1.85}$ & 66.93$_{\pm1.96}$ & 71.79$_{\pm2.35}$ & 78.04$_{\pm3.46}$ & \secondbest{78.28$_{\pm0.22}$} & 77.26$_{\pm0.49}$ & \best{79.12$_{\pm0.40}$}\\

\multirow{4}{*}{SbP ($\rho=90\%$)} & Aligned & 70.02$_{\pm2.20}$ & 96.51$_{\pm0.26}$ & 82.84$_{\pm4.47}$ & 87.36$_{\pm1.25}$ & 76.82$_{\pm2.80}$ & 68.57$_{\pm2.16}$ & 74.63$_{\pm4.61}$ & 83.38$_{\pm3.02}$\\
& Conflicting & 89.80$_{\pm0.87}$ & 31.21$_{\pm3.04}$ & 67.66$_{\pm5.12}$ & 63.58$_{\pm2.86}$ & 85.75$_{\pm0.32}$ & 87.46$_{\pm2.17}$ & 83.30$_{\pm3.96}$ & 77.31$_{\pm3.81}$\\
& Balanced & 79.94$_{\pm0.68}$ & 63.86$_{\pm1.39}$ & 75.26$_{\pm0.76}$ & 75.47$_{\pm0.95}$ & \best{80.49$_{\pm0.20}$} & 78.02$_{\pm0.18}$ & 78.96$_{\pm0.33}$ & \secondbest{80.34$_{\pm0.39}$}\\
\rowcolor[gray]{0.9}& Overall & 80.23$_{\pm0.37}$ & 69.84$_{\pm1.32}$ & 75.52$_{\pm0.73}$ & 76.25$_{\pm1.35}$ & \secondbest{78.78$_{\pm3.02}$} & 78.26$_{\pm0.18}$ & 78.76$_{\pm0.15}$ & \best{80.07$_{\pm0.21}$}\\
\hline

\multirow{4}{*}{GbP (case 1)} & Aligned & 85.81$_{\pm0.16}$ & 89.42$_{\pm0.25}$ & 86.88$_{\pm1.35}$ & 86.99$_{\pm0.56}$ & 90.17$_{\pm0.42}$ & 88.73$_{\pm1.34}$ & 86.12$_{\pm0.46}$ & 88.08$_{\pm0.45}$\\
& Conflicting & 83.96$_{\pm0.17}$ & 77.21$_{\pm0.33}$ & 83.43$_{\pm0.79}$ & 78.80$_{\pm1.09}$ & 77.07$_{\pm1.73}$ & 77.47$_{\pm0.09}$ & 77.92$_{\pm0.23}$ & 78.51$_{\pm0.59}$\\
& Balanced & 84.86$_{\pm0.05}$ & \secondbest{83.31$_{\pm0.05}$} & 80.00$_{\pm0.58}$ & 82.89$_{\pm0.80}$ & \best{83.62$_{\pm0.68}$} & 83.10$_{\pm0.64}$ & 82.02$_{\pm0.31}$ &83.30$_{\pm0.52}$\\
\rowcolor[gray]{0.9}& Overall & 84.93$_{\pm0.01}$ & \secondbest{83.75$_{\pm0.05}$} & 83.60$_{\pm0.87}$ & 83.16$_{\pm0.77}$ & \best{84.13$_{\pm0.56}$} & 83.46$_{\pm0.71}$ & 82.23$_{\pm0.30}$ &83.59$_{\pm0.53}$\\

\multirow{4}{*}{GbP (case 2)} & Aligned & 83.76$_{\pm1.59}$ & 89.39$_{\pm0.85}$ & 87.56$_{\pm0.77}$ & 89.30$_{\pm0.87}$ & 86.34$_{\pm0.64}$ & 87.25$_{\pm0.62}$ & 80.44$_{\pm0.58}$ & 88.80$_{\pm0.36}$\\
& Conflicting & 85.14$_{\pm0.31}$ & 76.13$_{\pm0.93}$ & 84.39$_{\pm0.52}$ & 80.82$_{\pm0.42}$ & 81.69$_{\pm2.67}$ & 79.07$_{\pm0.96}$ & 85.51$_{\pm0.57}$ &81.88$_{\pm1.33}$\\
& Balanced & 84.45$_{\pm0.65}$ & 82.76$_{\pm0.78}$ & 81.23$_{\pm1.01}$ & \secondbest{85.06$_{\pm0.23}$} & 84.02$_{\pm1.01}$ & 83.16$_{\pm0.45}$ & 82.98$_{\pm0.19}$ & \best{85.34$_{\pm0.48}$}\\
\rowcolor[gray]{0.9}& Overall & 84.42$_{\pm0.61}$ & 82.93$_{\pm0.78}$ & 84.40$_{\pm0.93}$ & \secondbest{85.20$_{\pm0.30}$} & 84.03$_{\pm0.97}$ & 83.19$_{\pm0.44}$ & 83.09$_{\pm0.21}$ &\best{85.44$_{\pm0.45}$}\\
\hline

\multirow{3}{*}{DbP} & w/ Drain & 87.99$_{\pm0.84}$ & 87.50$_{\pm0.64}$ & 87.19$_{\pm0.72}$ & 86.93$_{\pm0.84}$ & 87.01$_{\pm1.06}$ & 86.78$_{\pm0.48}$ & 87.31$_{\pm0.65}$ & 88.25$_{\pm0.31}$ \\
& w/o Drain & 77.32$_{\pm1.68}$ & 75.27$_{\pm2.03}$ & 75.87$_{\pm0.83}$ & 73.57$_{\pm0.93}$ & 76.90$_{\pm4.17}$ & 72.81$_{\pm0.52}$ & 74.60$_{\pm0.04}$ &  76.96$_{\pm1.32}$ \\
\rowcolor[gray]{0.9}& Overall & 82.60$_{\pm1.24}$ & 81.39$_{\pm1.31}$ & \secondbest{81.53$_{\pm0.72}$} & 80.25$_{\pm0.79}$ & 80.88$_{\pm0.65}$ & 79.79$_{\pm0.49}$ & 80.96$_{\pm0.31}$ & \best{82.61$_{\pm0.82}$} \\
\hline

\multirow{4}{*}{Ol3I} & Aligned & 71.53$_{\pm4.33}$ & 87.02$_{\pm1.13}$ & 78.13$_{\pm7.23}$ & 88.86$_{\pm1.19}$ & 86.92$_{\pm0.32}$ & 71.73$_{\pm4.69}$ & 89.46$_{\pm2.25}$ & 89.42$_{\pm2.21}$ \\

& Conflicting & 42.69$_{\pm1.83}$ & 34.52$_{\pm1.76}$ & 35.90$_{\pm4.90}$ & 31.17$_{\pm6.84}$ & 33.62$_{\pm5.76}$ & 62.07$_{\pm1.83}$ & 37.31$_{\pm3.82}$ &37.16$_{\pm4.97}$ \\

& Balanced & 57.11$_{\pm2.92}$ & 61.27$_{\pm0.66}$ & 57.01$_{\pm3.20}$ & 60.02$_{\pm3.43}$ & 60.27$_{\pm2.88}$ & \best{66.90$_{\pm1.43}$} & \secondbest{63.38$_{\pm0.78}$} &63.29$_{\pm1.87}$ \\

\rowcolor[gray]{0.9}& Overall & 62.05$_{\pm3.36}$ & 72.43$_{\pm0.96}$ & 64.52$_{\pm4.18}$ & 71.34$_{\pm2.16}$ & 71.21$_{\pm1.03}$ & 61.79$_{\pm1.03}$ & \best{74.27$_{\pm0.73}$} & \secondbest{74.15$_{\pm1.00}$} \\
\hline
\rowcolor[gray]{0.9}\multicolumn{2}{c}{Averaged Overall AUC} & 79.11 & 73.81 & 73.98 & 76.05 & \secondbest{78.82} & 76.66 & 78.44 & \best{80.42} \\

\toprule[1.5pt]
\end{tabular}
}
\label{exp:MDB_comparison}
\end{table*}

\subsubsection{Effects of Hyper-parameters}
We then show the effects of the proposed adaptive agreement learning and the bias council in Fig. \ref{fig:Ablation}.
The first and second rows illustrate the results on the SbP dataset with $\rho=99\%$ and the OL3I dataset, respectively.

We first set the number of classification heads fixed for the debiasing models, and evaluated the effects of $\lambda$.
Generally, $\lambda$ balances the learning on the bias-aligned samples and the bias-conflicting samples, as can be observed by the changes of aligned AUC and conflicting AUC from Figs. \ref{fig:lambda_SbP99_align_conflit} and \ref{fig:lambda_Ol3I_align_conflit}.
Also, the overall and balanced AUC would be harmed if $\lambda$ was set too small or large, as can be observed from Figs. \ref{fig:lambda_SbP99_overall_balanced} and \ref{fig:lambda_Ol3I_overall_balanced}.
Essentially, as the agree-disagree loss is applied on all samples,  $\lambda$ plays an important role in balancing the preference between learning agreement and learning disagreement.
Decreasing or increasing $\lambda$ would encourage the learning on the bias-aligned samples or the bias-conflicting samples, respectively.
This effect is more obvious on the SbP dataset where the ratio of biases in training and testing sets are controlled for better debiasing demonstration.
Then, we set $\lambda$ to a fixed value and varied the number of classification heads.
Generally, using a set of classifiers as a bias council in the biased model would help debiasing, as can be observed from Figs. \ref{fig:head_SbP99_overall_balanced} and \ref{fig:head_Ol3I_overall_balanced}.
Overall, it can be demonstrated that Ada-ABC can robustly learn from both the bias-aligned samples and the bias-conflicting samples and manage to mitigate the dataset biases.

\subsection{Comparative Study with MDB}
\subsubsection{Compared Approaches}
The compared methods include
i) the ERM model which is trained using cross entropy loss; 
ii) Group Distribution Robust Optimization (G-DRO) \cite{sagawa2020distributionally} which optimizes the performance of the worst-performing group with the knowledge of bias labels;
iii) D-BAT \cite{pagliardini2022agree}, an out-of-distribution generalization method that trains a set of different classifiers different from each other using an unlabeled OOD set. The validation sets are used as the OOD sets for this method.
iv) Just Train Twice (JTT) \cite{liu2021just}, a two-stage approach that first trains a biased ERM model and then develops the debiased model with a sampling ratio generated from the ERM model;
v) Pseudo Bias-Balanced Learning (PBBL) \cite{luo2022pseudo}, a two-stage method which estimates the Bayes distribution of biases and target labels first and then uses the prior for debiased model training;
(vi) Learning from Failure (LfF) \cite{nam2020learning}, a one-stage algorithm that developed debiased model with a loss-weighting strategy assisted by a highly biased model;
and
(vii) Disentangled Feature Augmentation (DFA) \cite{lee2021learning}, a one-stage method that further introduces feature disentanglement and augmentation into LfF.

\subsubsection{Implementation Details} 
As different datasets were with different bias scenarios, different value of $\lambda$ and number of heads were chosen, as shown in Table \ref{tab:hyper-parameters}.
Moreover, the hyper-parameter $q$ in the generalized cross entropy loss is set to 0.7 as recommended in \cite{zhang2018generalized}.

Our implementations used the PyTorch framework on a GeForce RTX™ 3090 GPU.
For SbP, GbP, and DbP, all methods were finetuned from DenseNet-121 \cite{huang2017densely}, whereas large-scale CXR pre-trained weights \cite{cohen2022torchxrayvision} were adopted for the CXR datasets.
For OL3I, experiments were conducted based on ResNet-18 with ImageNet pre-trained weights, following \cite{zong2023medfair}.
Adam \cite{kingma2014adam} with a learning rate of 1e-4 was used as the optimizer.
The results on the testing sets were obtained by the models with the best overall AUC on the validation set.
We constructed three runs for each method, and the averaged results as well as the standard deviation were reported.

\subsubsection{Quantitative Evaluation}

\textbf{Source-biased Pneumonia Classification.}
As can be observed in Table \ref{exp:MDB_comparison}, the ERM model could still be biased when finetuned on the biased dataset despite having been pre-trained on large-scale CXR dataset.
All debiasing algorithms mitigated dataset biases to certain levels.
Specifically, the one-stage methods LfF and DFA could prefer learning bias-conflicting samples.
On the other hand, our proposed Ada-ABC increased the AUC on bias-conflicting samples without a large sacrifice on the bias-aligned AUC.
As a result, Ada-ABC showed robust performance on all three cases with overall AUCs of 77.99\%, 79.12\%, and 80.07\% when $\rho=$ 99\%, 95\%, and 90\%, respectively, achieving consistent improvement compared with other methods that do not use bias label information.
It's worth noting that, Ada-ABC gradually achieved comparable performance to that of G-DRO with the decrease of $\rho$, demonstrating robust features learning capability.

\textbf{Gender-biased Pneumothorax Classification.}
By Table \ref{exp:MDB_comparison}, our proposed Ada-ABC showed high AUC on the bias-aligned samples with increased performance on the bias-conflicting samples.
As a result, Ada-ABC achieved 83.59\% and 85.44\% overall AUC under cases of GbP1 and GbP2, respectively. 
Notably, for GbP2, Ada-ABC could even surpass G-DRO on the balanced AUC and Overall AUC, showing an effective and robust solution to the gender bias in medical image classification.

\textbf{Drain-biased Pneumothorax Classification.}
As healthy cases do not have chest drains, we computed the AUC with healthy samples and pneumothorax cases with or without chest drains.
We also provided the overall AUC.
As reported in Table \ref{exp:MDB_comparison}, there is a clear performance gap between the AUCs computed on cases with or without chest drains, showing that it was much easier for the models to learn to distinguish pneumothorax with chest drains.
We found that other methods (except G-DRO and D-BAT) may perform even worse than the ERM model, which indicates that the dataset bias in this case is harder to mitigate.
In contrast, our proposed Ada-ABC demonstrated its effectiveness by achieving the best performance for all cases, with or without chest drains, showing that it learned more bias-invariant features.

\textbf{Age-biased Ischemic Heart Disease Prognosis.}
The OL3I dataset is not only biased but also highly imbalanced, posing a hard challenge to the debiasing algorithms. 
In our experiment, most of the compared algorithms were not even comparable to the ERM model.
The disadvantage of LfF was also amplified here, where much weight was put on the high-loss samples by the GCE model, yet the large proportion of low-loss samples was ignored, leading to high conflicting AUC and low aligned AUC.
Notably, G-DRO achieved a low overall AUC, which is in line with the findings by Zong \etal \cite{zong2023medfair}, and we deemed that it also over-weighed the minority groups.
In this dataset, DFA achieved the best overall AUC, mostly due to that its feature augmentation also helped alleviate the class imbalance.
The proposed Ada-ABC achieved the second-best overall AUC with a marginal difference compared to DFA, demonstrating robust feature learning under scenarios with more complex sub-group distribution shifts.

\textbf{Overall Results.}
The last row of Table \ref{exp:MDB_comparison} reports the average of the mean of overall AUC across the seven scenarios.
Our proposed Ada-ABC achieved 80.42\% averaged result on the medical debiasing benchmark, with clear improvement compared with either the two-stage or one-stage debiasing approaches.
We found that Ada-ABC could even surpass the performance of G-DRO, showing robust feature learning capability and consistent improvement.

\subsubsection{Qualitative Visualization} We visualize the saliency maps \cite{zhou2016learning} of the ERM model and the debiasing model developed with Ada-ABC on the four datasets in Fig. \ref{Fig:Saliency}.
In particular, $t=b=1$ for all the images shown here, and both models gave correct predictions.
However, the ERM model tended to use the wrong regions to identify the patient.
In contrast, the debiasing model developed using Ada-ABC could successfully mitigate the bias on the shown samples, attending to the correct regions corresponding to the disease signs.
In other words, our proposed model could learn to make the right decisions for the right reasons.
This observation further highlights the significance of addressing dataset bias for robust and trustworthy medical image analysis.

\begin{figure}[t]
\centering
\includegraphics[width=\linewidth]{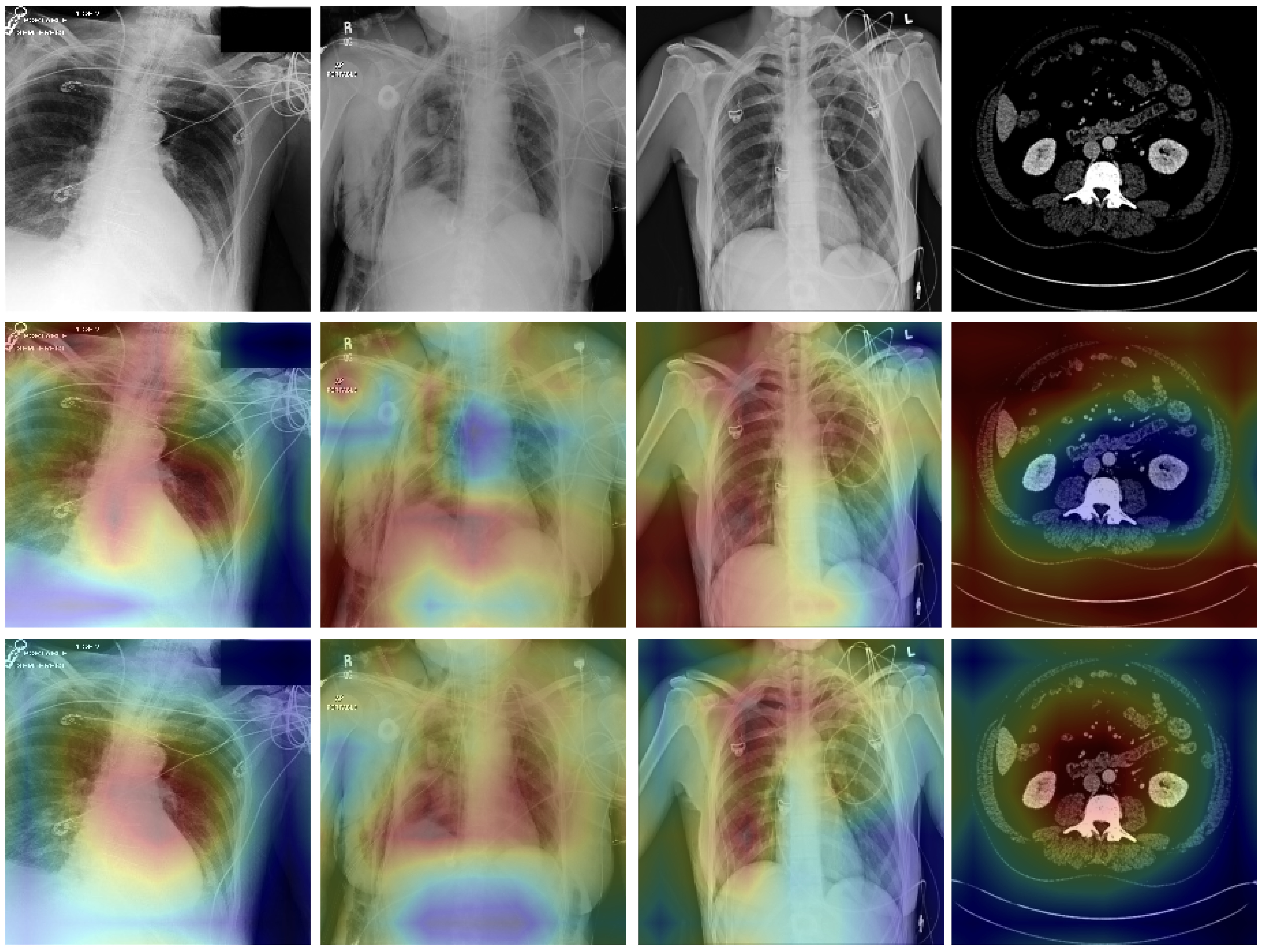}    
\caption{The saliency maps by the ERM model (2nd row) and the debiasing model by Ada-ABC (3rd row). Samples from columns 1-4 are from SbP, GbP, DbP, and OL3I, respectively. Both models made correct predictions but were looking for different reasons.}
\label{Fig:Saliency}
\end{figure}

\section{Conclusion}
In summary, this paper proposes a simple yet effective one-stage debiasing framework, \textbf{Ada}ptive \textbf{A}greement from \textbf{B}iased \textbf{C}ouncil (Ada-ABC).
Ada-ABC is based on simultaneous training of a biased network and a debiasing network.
The biased model is developed to capture the bias information in the dataset, using a bias council trained with the generalized cross entropy loss to amplify the learning preference on the samples with spurious correlation.
Then, the debiasing model adaptively learns to agree or disagree with the biased model on the samples with or without spurious correlation, respectively, under the supervision of our proposed adaptive learning loss.
We provided theoretical analysis to prove that the debiasing model could learn the targeted feature when the biased model successfully captures the bias information.
Further, we constructed the first medical debiasing benchmark (MBD) to our best knowledge, which consists of four datasets with seven different bias scenarios.
Based on MBD, we validated the effectiveness of Ada-ABC in mitigating dataset bias with extensive experiments and showed that it consistently achieves state-of-the-art performance in most of the studied cases.
We demonstrated that, both theoretically and practically, Ada-ABC provides a promising way for more accurate, fair, and trustworthy medical image analysis.

\bibliographystyle{sss.bst}
\bibliography{refs}

\begin{thebibliography}{10}

\bibitem{topol2019high}
E.~J. Topol, ``High-performance medicine: the convergence of human and artificial intelligence,'' {\em Nature medicine}, vol.~25, no.~1, pp.~44--56, 2019.

\bibitem{geirhos2020shortcut}
R.~Geirhos {\em et~al.}, ``Shortcut learning in deep neural networks,'' {\em Nature Machine Intelligence}, vol.~2, no.~11, pp.~665--673, 2020.

\bibitem{oakden2020hidden}
L.~Oakden-Rayner {\em et~al.}, ``Hidden stratification causes clinically meaningful failures in machine learning for medical imaging,'' in {\em Proceedings of the ACM conference on health, inference, and learning}, pp.~151--159, 2020.

\bibitem{luo2022pseudo}
L.~Luo {\em et~al.}, ``Pseudo bias-balanced learning for debiased chest x-ray classification,'' in {\em Medical Image Computing and Computer Assisted Intervention--MICCAI 2022: 25th International Conference, Singapore, September 18--22, 2022, Proceedings, Part VIII}, pp.~621--631, Springer, 2022.

\bibitem{degrave2021ai}
A.~J. DeGrave, J.~D. Janizek and S.-I. Lee, ``Ai for radiographic covid-19 detection selects shortcuts over signal,'' {\em Nature Machine Intelligence}, vol.~3, no.~7, pp.~610--619, 2021.

\bibitem{larrazabal2020gender}
A.~J. Larrazabal {\em et~al.}, ``Gender imbalance in medical imaging datasets produces biased classifiers for computer-aided diagnosis,'' {\em Proceedings of the National Academy of Sciences}, vol.~117, no.~23, pp.~12592--12594, 2020.

\bibitem{gichoya2022ai}
J.~W. Gichoya {\em et~al.}, ``Ai recognition of patient race in medical imaging: a modelling study,'' {\em The Lancet Digital Health}, vol.~4, no.~6, pp.~e406--e414, 2022.

\bibitem{bluemke2020assessing}
D.~A. Bluemke {\em et~al.}, ``Assessing radiology research on artificial intelligence: a brief guide for authors, reviewers, and readers—from the radiology editorial board,'' 2020.

\bibitem{taylor2022uk}
S.~Taylor-Phillips {\em et~al.}, ``Uk national screening committee's approach to reviewing evidence on artificial intelligence in breast cancer screening,'' {\em The Lancet Digital Health}, vol.~4, no.~7, pp.~e558--e565, 2022.

\bibitem{nam2020learning}
J.~Nam {\em et~al.}, ``Learning from failure: De-biasing classifier from biased classifier,'' {\em Advances in Neural Information Processing Systems}, vol.~33, pp.~20673--20684, 2020.

\bibitem{shah2020pitfalls}
H.~Shah {\em et~al.}, ``The pitfalls of simplicity bias in neural networks,'' {\em Advances in Neural Information Processing Systems}, vol.~33, pp.~9573--9585, 2020.

\bibitem{kalimeris2019sgd}
D.~Kalimeris {\em et~al.}, ``Sgd on neural networks learns functions of increasing complexity,'' {\em Advances in neural information processing systems}, vol.~32, 2019.

\bibitem{hermann2020shapes}
K.~Hermann and A.~Lampinen, ``What shapes feature representations? exploring datasets, architectures, and training,'' {\em Advances in Neural Information Processing Systems}, vol.~33, pp.~9995--10006, 2020.

\bibitem{teney2022evading}
D.~Teney {\em et~al.}, ``Evading the simplicity bias: Training a diverse set of models discovers solutions with superior ood generalization,'' in {\em Proceedings of the IEEE/CVF Conference on Computer Vision and Pattern Recognition}, pp.~16761--16772, 2022.

\bibitem{rueckel2020impact}
J.~Rueckel {\em et~al.}, ``Impact of confounding thoracic tubes and pleural dehiscence extent on artificial intelligence pneumothorax detection in chest radiographs,'' {\em Investigative Radiology}, vol.~55, no.~12, pp.~792--798, 2020.

\bibitem{rouzrokh2022mitigating}
P.~Rouzrokh {\em et~al.}, ``Mitigating bias in radiology machine learning: 1. data handling,'' {\em Radiology: Artificial Intelligence}, vol.~4, no.~5, p.~e210290, 2022.

\bibitem{li2019repair}
Y.~Li and N.~Vasconcelos, ``Repair: Removing representation bias by dataset resampling,'' in {\em Proceedings of the IEEE/CVF Conference on Computer Vision and Pattern Recognition}, pp.~9572--9581, 2019.

\bibitem{sagawa2020distributionally}
S.~Sagawa {\em et~al.}, ``Distributionally robust neural networks for group shifts: On the importance of regularization for worst-case generalization,'' in {\em International Conference on Learning Representations}, 2020.

\bibitem{arjovsky2019invariant}
M.~Arjovsky {\em et~al.}, ``Invariant risk minimization,'' {\em arXiv preprint arXiv:1907.02893}, 2019.

\bibitem{zhang2021quantifying}
G.~Zhang {\em et~al.}, ``Quantifying and improving transferability in domain generalization,'' {\em Advances in Neural Information Processing Systems}, vol.~34, pp.~10957--10970, 2021.

\bibitem{zhou2022sparse}
X.~Zhou {\em et~al.}, ``Sparse invariant risk minimization,'' in {\em International Conference on Machine Learning}, pp.~27222--27244, PMLR, 2022.

\bibitem{tartaglione2021end}
E.~Tartaglione, C.~A. Barbano and M.~Grangetto, ``End: Entangling and disentangling deep representations for bias correction,'' in {\em Proceedings of the IEEE/CVF Conference on Computer Vision and Pattern Recognition}, pp.~13508--13517, 2021.

\bibitem{zhu2021learning}
W.~Zhu {\em et~al.}, ``Learning bias-invariant representation by cross-sample mutual information minimization,'' in {\em Proceedings of the IEEE/CVF International Conference on Computer Vision}, pp.~15002--15012, 2021.

\bibitem{sohoni2020no}
N.~Sohoni {\em et~al.}, ``No subclass left behind: Fine-grained robustness in coarse-grained classification problems,'' {\em Advances in Neural Information Processing Systems}, vol.~33, pp.~19339--19352, 2020.

\bibitem{liu2021just}
E.~Z. Liu {\em et~al.}, ``Just train twice: Improving group robustness without training group information,'' in {\em International Conference on Machine Learning}, pp.~6781--6792, PMLR, 2021.

\bibitem{lee2021learning}
J.~Lee {\em et~al.}, ``Learning debiased representation via disentangled feature augmentation,'' {\em Advances in Neural Information Processing Systems}, vol.~34, pp.~25123--25133, 2021.

\bibitem{kim2021biaswap}
E.~Kim, J.~Lee and J.~Choo, ``Biaswap: Removing dataset bias with bias-tailored swapping augmentation,'' in {\em Proceedings of the IEEE/CVF International Conference on Computer Vision}, pp.~14992--15001, 2021.

\bibitem{luo2022rethinking}
L.~Luo {\em et~al.}, ``Rethinking annotation granularity for overcoming shortcuts in deep learning--based radiograph diagnosis: A multicenter study,'' {\em Radiology: Artificial Intelligence}, vol.~4, no.~5, p.~e210299, 2022.

\bibitem{viviano2020saliency}
J.~D. Viviano {\em et~al.}, ``Saliency is a possible red herring when diagnosing poor generalization,'' in {\em ICLR}, 2020.

\bibitem{seyyed2020chexclusion}
L.~Seyyed-Kalantari {\em et~al.}, ``Chexclusion: Fairness gaps in deep chest x-ray classifiers,'' in {\em BIOCOMPUTING 2021: proceedings of the Pacific symposium}, pp.~232--243, World Scientific, 2020.

\bibitem{seyyed2021underdiagnosis}
L.~Seyyed-Kalantari {\em et~al.}, ``Underdiagnosis bias of artificial intelligence algorithms applied to chest radiographs in under-served patient populations,'' {\em Nature medicine}, vol.~27, no.~12, pp.~2176--2182, 2021.

\bibitem{zufiria2022analysis}
B.~Zufiria {\em et~al.}, ``Analysis of potential biases on mammography datasets for deep learning model development,'' in {\em Applications of Medical Artificial Intelligence: First International Workshop, AMAI 2022, Held in Conjunction with MICCAI 2022, Singapore, September 18, 2022, Proceedings}, pp.~59--67, Springer, 2022.

\bibitem{zhao2020training}
Q.~Zhao, E.~Adeli and K.~M. Pohl, ``Training confounder-free deep learning models for medical applications,'' {\em Nature communications}, vol.~11, no.~1, p.~6010, 2020.

\bibitem{hong2021unbiased}
Y.~Hong and E.~Yang, ``Unbiased classification through bias-contrastive and bias-balanced learning,'' {\em Advances in Neural Information Processing Systems}, vol.~34, 2021.

\bibitem{kim2022learning}
N.~Kim {\em et~al.}, ``Learning debiased classifier with biased committee,'' in {\em Advances in Neural Information Processing Systems}, 2022.

\bibitem{arpit2017closer}
D.~Arpit {\em et~al.}, ``A closer look at memorization in deep networks,'' in {\em International conference on machine learning}, pp.~233--242, PMLR, 2017.

\bibitem{pagliardini2022agree}
M.~Pagliardini {\em et~al.}, ``Agree to disagree: Diversity through disagreement for better transferability,'' in {\em ICLR}, 2023.

\bibitem{zhang2018generalized}
Z.~Zhang and M.~Sabuncu, ``Generalized cross entropy loss for training deep neural networks with noisy labels,'' {\em Advances in neural information processing systems}, vol.~31, 2018.

\bibitem{nam2021diversity}
G.~Nam {\em et~al.}, ``Diversity matters when learning from ensembles,'' {\em Advances in neural information processing systems}, vol.~34, pp.~8367--8377, 2021.

\bibitem{luo2020deep}
L.~Luo {\em et~al.}, ``Deep mining external imperfect data for chest x-ray disease screening,'' {\em IEEE transactions on medical imaging}, vol.~39, no.~11, pp.~3583--3594, 2020.

\bibitem{johnson2019mimic}
A.~E. Johnson {\em et~al.}, ``Mimic-cxr, a de-identified publicly available database of chest radiographs with free-text reports,'' {\em Scientific data}, vol.~6, no.~1, pp.~1--8, 2019.

\bibitem{wang2017chestx}
X.~Wang {\em et~al.}, ``Chestx-ray8: Hospital-scale chest x-ray database and benchmarks on weakly-supervised classification and localization of common thorax diseases,'' in {\em Proceedings of the IEEE conference on computer vision and pattern recognition}, pp.~2097--2106, 2017.

\bibitem{zambrano2023opportunistic}
J.~M. Zambrano~Chaves {\em et~al.}, ``Opportunistic assessment of ischemic heart disease risk using abdominopelvic computed tomography and medical record data: a multimodal explainable artificial intelligence approach,'' {\em Scientific Reports}, vol.~13, no.~1, p.~21034, 2023.

\bibitem{zong2023medfair}
Y.~Zong, Y.~Yang and T.~Hospedales, ``Medfair: Benchmarking fairness for medical imaging,'' in {\em The Eleventh International Conference on Learning Representations}, 20223.

\bibitem{huang2017densely}
G.~Huang {\em et~al.}, ``Densely connected convolutional networks,'' in {\em Proc. IEEE Conf. Comput. Vis. Pattern Recognit.}, pp.~4700--4708, 2017.

\bibitem{cohen2022torchxrayvision}
J.~P. Cohen {\em et~al.}, ``Torchxrayvision: A library of chest x-ray datasets and models,'' in {\em International Conference on Medical Imaging with Deep Learning}, pp.~231--249, PMLR, 2022.

\bibitem{kingma2014adam}
D.~P. Kingma and J.~Ba, ``Adam: A method for stochastic optimization,'' in {\em Proc. Int. Conf. Learn. Representations}, 2015.

\bibitem{zhou2016learning}
B.~Zhou {\em et~al.}, ``Learning deep features for discriminative localization,'' in {\em CVPR}, pp.~2921--2929, 2016.

\end{thebibliography}

\end{document}